\documentclass[10pt,twocolumn]{article}

\usepackage{microtype}
\usepackage{graphicx}
\usepackage{subcaption}
\usepackage{booktabs} 
\usepackage{tabularx} 
\usepackage{multirow} 
\usepackage{longtable} 
\usepackage{algorithm}
\usepackage{algorithmic}

\usepackage{amsmath}
\usepackage{amssymb}
\usepackage{mathtools}
\usepackage{amsthm}

\usepackage[colorlinks=true,linkcolor=blue,citecolor=blue,urlcolor=blue]{hyperref}
\usepackage[capitalize,noabbrev]{cleveref}

\theoremstyle{plain}
\newtheorem{theorem}{Theorem}[section]
\newtheorem{proposition}[theorem]{Proposition}
\newtheorem{lemma}[theorem]{Lemma}
\newtheorem{corollary}[theorem]{Corollary}
\theoremstyle{definition}
\newtheorem{definition}[theorem]{Definition}

\theoremstyle{remark}

\usepackage{tikz}
\usetikzlibrary{shapes.geometric,arrows.meta,positioning,patterns}

\title{Geometric Manifold Rectification for Imbalanced Learning}
\author{
Xubin Wang$^{1}$ \and Qing Li$^{2}$\thanks{IEEE Fellow; Chair Professor (Data Science) and Head of the Department of Computing.} \and Weijia Jia$^{1}$\thanks{IEEE Fellow; Chair Professor and Head of the BNU-BNBU Institute of Artificial Intelligence and Future Networks.}\\[0.4em]
$^{1}$BNU-BNBU Institute of Artificial Intelligence and Future Networks,\\
Beijing Normal-Hong Kong Baptist University and Beijing Normal University at Zhuhai\\
$^{2}$Department of Computing, The Hong Kong Polytechnic University\\[0.3em]
\texttt{wangxubin@ieee.org}\\
\texttt{qing-prof.li@polyu.edu.hk}\\
\texttt{jiawj@bnu.edu.cn}
}
\date{}

\begin{document}
\maketitle

\begin{abstract}
Imbalanced classification presents a formidable challenge in machine learning, particularly when tabular datasets are plagued by noise and overlapping class boundaries. From a geometric perspective, the core difficulty lies in the topological intrusion of the majority class into the minority manifold, which obscures the true decision boundary. Traditional undersampling techniques, such as Edited Nearest Neighbours (ENN), typically employ symmetric cleaning rules and uniform voting, failing to capture the local manifold structure and often inadvertently removing informative minority samples. In this paper, we propose GMR (Geometric Manifold Rectification), a novel framework designed to robustly handle imbalanced structured data by exploiting local geometric priors. GMR makes two contributions: (1) \textit{Geometric confidence estimation} that uses inverse-distance weighted $k$NN voting with an adaptive distance metric to capture local reliability; and (2) \textit{asymmetric cleaning} that is strict on majority samples while conservatively protecting minority samples via a safe-guarding cap on minority removal.  Extensive experiments on multiple benchmark datasets show that GMR is competitive with strong sampling baselines.
\end{abstract}

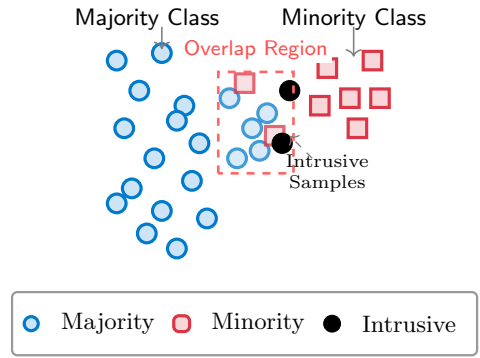
\begin{figure}[!t]
\centering
\begin{tikzpicture}[scale=1.0]
  \definecolor{majColor}{RGB}{0, 122, 204}      
  \definecolor{minColor}{RGB}{220, 53, 69}      
  \definecolor{removedColor}{RGB}{128, 128, 128} 
  
  \tikzset{
    majorNode/.style={circle, draw=majColor, fill=majColor!20, line width=1.2pt, inner sep=2.5pt},
    minorNode/.style={rectangle, draw=minColor, fill=minColor!20, line width=1.2pt, minimum width=7pt, minimum height=7pt},
    intrusiveNode/.style={circle, draw=black, fill=black, inner sep=3pt},
    labelText/.style={font=\small\sffamily, align=center},
    legendBox/.style={draw=black!40, fill=white, line width=0.85pt, rounded corners=2pt, inner sep=4pt}
  }
  
  \foreach \x/\y in {
    0.5/3.8, 0.8/3.4, 1.1/3.9, 1.4/3.2, 
    0.6/2.9, 1.0/2.5, 1.3/3.0, 1.6/2.7,
    0.7/2.1, 1.1/1.8, 1.5/2.2, 0.9/1.5,
    1.3/1.3, 0.5/1.9, 1.7/1.7} {
    \node[majorNode] at (\x, \y) {};
  }
  
  \foreach \x/\y in {
    3.3/3.7, 3.6/3.3, 3.9/3.8, 
    3.2/3.2, 3.7/2.9, 4.0/3.3} {
    \node[minorNode] at (\x, \y) {};
  }

  \foreach \x/\y in {2.0/3.3, 2.3/2.9, 2.1/2.5, 2.5/3.1, 2.4/2.6} {
    \node[majorNode, opacity=0.75] at (\x, \y) {};
  }
  \foreach \x/\y in {2.2/3.5, 2.6/2.8} {
    \node[minorNode, opacity=0.75] at (\x, \y) {};
  }

  \node[intrusiveNode, scale=0.9] (int1) at (2.8, 3.4) {};
  \node[intrusiveNode, scale=0.9] (int2) at (2.7, 2.7) {};
  
  \draw[dashed, line width=1pt, red!60] (1.85, 2.3) rectangle (2.85, 3.65);
  \node[font=\footnotesize\sffamily, fill=white, inner sep=1.5pt] at (2.35, 3.95) {\textcolor{red!70}{Overlap Region}};
  
  \node[labelText] at (0.9, 4.35) {Majority Class};
  \draw[->, line width=0.6pt, black!50] (1.1, 4.25) -- (1.1, 3.95);
  
  \node[labelText] at (3.65, 4.35) {Minority Class};
  \draw[->, line width=0.6pt, black!50] (3.65, 4.25) -- (3.65, 3.85);
  
  \node[labelText, font=\scriptsize, text width=1.8cm] at (3.3, 2.3) {Intrusive\\Samples};
  \draw[->, line width=0.5pt, black!50, dashed] (3.1, 2.5) -- (2.8, 2.8);
  
  \node[legendBox] at (2.2, 0.3) {
    \begin{tikzpicture}[baseline=(current bounding box.center)]
      \node[majorNode, scale=0.75] at (0, 0) {};
      \node[font=\small, anchor=west] at (0.25, 0) {Majority};
      
      \node[minorNode, scale=0.75] at (2.0, 0) {};
      \node[font=\small, anchor=west] at (2.25, 0) {Minority};
      
      \node[intrusiveNode, scale=0.75] at (4.0, 0) {};
      \node[font=\small, anchor=west] at (4.25, 0) {Intrusive};
    \end{tikzpicture}
  };

\end{tikzpicture}
\caption{Illustration of the core challenges in imbalanced classification. The majority class (blue circles) significantly outnumbers the minority class (red squares). The dashed box highlights the overlap region where samples from both classes intermingle, creating ambiguity. Black dots mark intrusive majority samples that have penetrated the minority manifold—these ambiguous boundary samples degrade classifier performance and are the primary targets of GMR's geometric cleaning strategy.}
\label{fig:motivation}
\end{figure}

\section{Introduction}
Imbalanced data distribution is a pervasive and critical challenge in machine learning, particularly in the domain of tabular data analysis. This issue manifests in high-stakes applications such as credit card fraud detection, rare disease diagnosis, and network intrusion monitoring, where data is naturally structured and features are heterogeneous \cite{kaur2019systematic}. In these scenarios, the class of interest (minority class) is significantly outnumbered by the normal class (majority class). While deep learning has revolutionized perceptual tasks like image and text classification, traditional machine learning algorithms (e.g., Gradient Boosting, SVM, kNN) remain highly competitive for tabular data \cite{grinsztajn2022tree}. However, these standard algorithms, which typically optimize for global accuracy, inherently favor the majority class \cite{he2009learning}. This bias results in suboptimal predictive performance for the minority class, which is often the primary target of the learning task.

The difficulty of imbalanced learning is further exacerbated by data complexities such as class overlap and label noise. As illustrated in Figure~\ref{fig:motivation}, the overlap region between majority and minority manifolds creates a critical challenge: intrusive majority samples (shown as black dots) penetrate the minority space, while ambiguous samples near the boundary introduce uncertainty. Prior literature suggests that the absolute imbalance ratio is not the sole factor hindering performance; noisy and overlapping boundary samples are often the dominant failure mode in practice \cite{he2009learning,kaur2019systematic}. When majority class samples invade the minority class space (overlap) or when labels are incorrect (noise), the decision boundary becomes ambiguous. This geometric perspective reveals that effective imbalanced learning requires not only ratio adjustment but also topology-aware data rectification.

Data-level solutions, particularly resampling techniques, are the most common approach to mitigate these issues \cite{he2009learning}. Undersampling methods aim to balance the class distribution by reducing the number of majority samples. Random Under Sampling (RUS) is simple but risks discarding informative data. More sophisticated cleaning methods, such as Edited Nearest Neighbours (ENN) and Condensed Nearest Neighbour (CNN), attempt to selectively remove samples to clarify the boundary. However, as we illustrate through Figure~\ref{fig:motivation}, these traditional cleaning methods suffer from two fundamental limitations:

\begin{enumerate}
    \item \textbf{Symmetric Treatment of Asymmetric Costs}: Traditional methods typically apply symmetric cleaning rules, removing any sample that is misclassified by its neighbors, regardless of its class. In Figure~\ref{fig:motivation}, this would mean treating the intrusive majority samples (black dots in the overlap region) and potential minority outliers with equal aggressiveness. In imbalanced domains, this symmetry is flawed. Minority samples are scarce and carry high information value; inadvertently removing a minority sample (false positive cleaning) is far more detrimental than retaining a noisy majority sample. In practice, robust methods should be ``strict'' on the majority class (aggressively removing the black intrusive samples) but ``protective'' of the minority class (conservatively preserving the red squares even near boundaries).
    
    \item \textbf{Neglect of Local Spatial Distribution}: Most neighbor-based cleaning methods rely on uniform majority voting within a $k$-nearest neighbor ($k$-NN) neighborhood. This approach assumes that all $k$ neighbors contribute equally to the local concept. In reality, as suggested by the spatial arrangement in Figure~\ref{fig:motivation}, neighbors that are spatially closer to the query sample should provide stronger evidence of its true label than distant neighbors. The concentration of majority samples (blue circles) near the minority cluster suggests varying degrees of intrusion---closer intrusive samples are more problematic than distant ones. Ignoring this distance information leads to coarse confidence estimates, especially in sparse or boundary regions like the marked overlap area.
\end{enumerate}

To address these limitations, we propose \textbf{GMR} (Geometric Manifold Rectification), a novel undersampling framework designed for robust imbalanced classification. GMR fundamentally rethinks the cleaning process by incorporating geometric priors and cost asymmetry. First, we introduce a \textit{Geometric Confidence Estimation} mechanism that weighs neighbors by their inverse distance, effectively capturing the local manifold topology to provide a fine-grained measure of sample reliability. We adopt adaptive metric selection: Euclidean distance for low-dimensional data and cosine similarity for high-dimensional data, ensuring robust distance computation across feature spaces. Second, we implement a \textit{Class-Aware Asymmetric Cleaning} strategy with strict majority cleaning and conservative minority protection. This asymmetric policy eliminates topological intrusion while preserving minority manifold integrity.

The main contributions of this paper are summarized as follows:
\begin{itemize}
    \item \textbf{Geometric Framework for Imbalanced Learning}: We propose a novel framework that integrates geometric priors into the resampling process, moving beyond simple distance heuristics to respect the underlying data manifold.
    \item \textbf{Theoretical Analysis}: We provide theoretical results that motivate our design choices under standard assumptions, including: (a) conditions under which inverse-distance weighting reduces estimation error compared to uniform voting (Theorem~\ref{thm:geometric_optimality}), (b) the posterior-shifting effect induced by asymmetric cleaning (Lemma~\ref{lem:cleaning_effect}), and (c) decision boundary alignment with asymmetric risk objectives (Proposition~\ref{prop:boundary_alignment}).
    \item \textbf{Asymmetric Manifold Preservation}: We introduce an asymmetric cleaning strategy that selectively prunes the majority class to resolve overlap while strictly preserving the topological structure of the minority class.
  \item \textbf{Extensive Empirical Validation}: We conduct experiments on 27 benchmark datasets with 7 classifiers under repeated train--test splits with 5 random seeds, demonstrating GMR's classifier-agnostic effectiveness. 
\end{itemize}

\textbf{Addressing Potential Concerns:} We anticipate reviewers may question: (1) \textit{Why not just use cost-sensitive learning?} Answer: Cost-sensitive methods reweight loss functions but do not remove noisy boundary samples---GMR complements them by cleaning the data space first. (2) \textit{Hyperparameter sensitivity?} We use a small set of intuitive hyperparameters with fixed defaults. (3) \textit{Computational cost?} The main overhead is nearest-neighbor search for confidence estimation; in practice this can be implemented efficiently with standard nearest-neighbor backends.


\section{Related Work}\label{sec:related_work}

Imbalanced classification has been extensively studied, with solutions spanning data-level resampling, algorithm-level reweighting, and hybrids. Recent work increasingly emphasizes the coupled effects of imbalance, overlap, and label noise, which motivates geometry-aware preprocessing.

\subsection{Data-Level Resampling Strategies}
Resampling methods aim to construct a balanced training set by modifying the class distribution.

\subsubsection{Oversampling and Synthesis}
SMOTE~\cite{chawla2002smote} and its variants (e.g., Borderline-SMOTE~\cite{han2005borderline}, ADASYN~\cite{he2008adasyn}) synthesize minority samples to balance class priors, while later methods improve synthesis quality via stronger generators or diversity controls~\cite{mullick2019generative,bej2021loras}. A key practical challenge is preventing synthesis from amplifying overlap/noise; several methods therefore couple generation with filtering or complexity-aware constraints~\cite{xu2022synthetic}. In contrast, GMR focuses on cleaning/undersampling: we aim to reduce overlap by removing unreliable boundary samples using geometric confidence and asymmetric, class-aware rules.

\subsubsection{Undersampling and Cleaning}
Undersampling reduces majority dominance by discarding selected majority instances. Classical neighborhood-based cleaning, such as ENN~\cite{wilson1972asymptotic}, Tomek Links~\cite{tomek1976two}, and NCR~\cite{laurikkala2001improving}, removes samples using local consistency tests.
\textit{Limitation:} Most cleaning rules are \textit{symmetric} across classes and rely on uniform $k$-NN voting, which can over-remove scarce minority samples and behave unstably in sparse overlap regions.

\subsection{Cost-Sensitive and Deep Imbalance Learning}
Algorithm-level methods modify the objective to emphasize minority performance, e.g., via reweighting and margin-aware losses (Focal Loss~\cite{lin2017focal}, LDAM-DRW~\cite{cao2019ldam}); related lines include decoupled training~\cite{kang2019decoupling}. This view is also connected to large-scale and multi-task feature selection under high-dimensional settings~\cite{wang2022self,wang2024mel}.
\textit{Limitation:} Many approaches are deep-learning-specific and can be sensitive to label noise (hard examples may be noisy), and they are less directly applicable to non-differentiable tabular models (e.g., GBDT/RF) that remain strong in practice~\cite{grinsztajn2022tree}. GMR instead provides model-agnostic data rectification prior to training.

\subsection{Learning with Label Noise}
Learning with label noise (LLN) is commonly addressed by training-time filtering or reweighting, often leveraging early-learning dynamics (e.g., Co-teaching~\cite{han2018coteaching}, DivideMix~\cite{li2020dividemix}, MentorNet~\cite{jiang2018mentornet}, ELR~\cite{liu2020elr}).
\textit{Limitation:} Under class imbalance, the small-loss/early-learning heuristic can discard hard-but-clean minority samples, while noisy outliers may be over-emphasized; robust methods often require complex end-to-end training assumptions. GMR targets a lightweight, model-agnostic alternative by using local geometric structure to identify unreliable boundary instances before training.

\subsection{Geometric Priors in Classification}
Geometric and topological data analysis (TDA) has shown promise in understanding complex data structures. Manifold learning techniques assume that high-dimensional data lies on low-dimensional manifolds embedded in high-dimensional ambient space. In imbalanced learning, preserving the topology of the minority manifold is crucial for maintaining class separability. However, most existing resampling methods use Euclidean distance in a heuristic manner (e.g., uniform $k$-NN voting) without explicitly modeling the local manifold density, curvature, or topological constraints. This limitation is particularly severe in overlap regions where local geometry is highly curved and class boundaries are ambiguous.
Geometric priors have also been used in other tasks to stabilize learning under ambiguity, supporting the general utility of injecting local geometric context.
\textit{Contribution:} GMR bridges these gaps by integrating geometric confidence estimation with an asymmetric cleaning policy, offering three key advantages: (1) \textbf{Geometric Awareness}: Unlike symmetric cleaning methods (ENN, NCR), GMR respects the unequal costs of error through class-dependent thresholds ($\alpha$ for majority, $\beta$ for minority with $\beta > \alpha$), preventing minority manifold collapse. (2) \textbf{Manifold Sensitivity}: Unlike uniform voting, GMR utilizes kernel-based density estimation with adaptive distance metrics to capture local manifold structure and mitigate the curse of dimensionality. (3) \textbf{Model Agnosticism}: Unlike deep learning-specific losses, GMR serves as a universal data preprocessor compatible with any downstream classifier (tree-based, linear, or neural). This modular design enables seamless integration into existing machine learning pipelines without requiring end-to-end retraining.

\section{Proposed Method}\label{sec:method}
Figure~\ref{fig:framework} provides an illustrative example of the GMR process: (a) displays an imbalanced dataset with the majority class intruding into the minority manifold; (b) visualizes the geometric voting scheme, where closer neighbors exert larger influence (thicker arrows) when estimating confidence; and (c) shows the result after GMR cleaning, with ambiguous majority samples removed and the effective decision boundary shifted away from the minority manifold. This illustrative triptych is intended to clarify the intuition behind GMR without resorting to a flowchart representation.

In this section, we present the proposed Geometric Manifold Rectification (GMR) framework. GMR is a data-level solution designed to address the dual challenges of class imbalance and label noise. Unlike traditional undersampling methods that rely on symmetric cleaning and uniform voting, GMR incorporates a geometric probabilistic mechanism and a cost-sensitive asymmetric cleaning strategy. We first formalize the problem, then provide a theoretical analysis based on Bayesian risk minimization and kernel density estimation, followed by the detailed algorithmic framework and its theoretical properties.

\begin{figure*}[t]
\centering

\includegraphics[width=\textwidth]{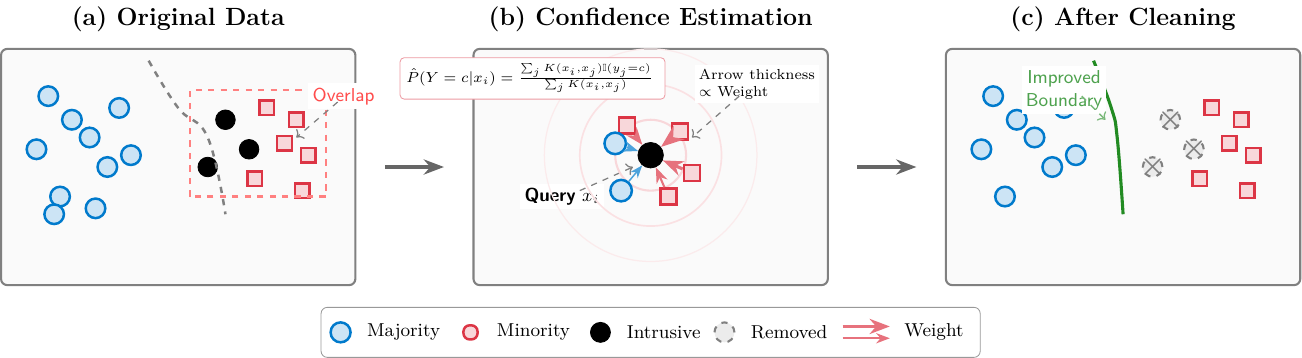}
\caption{Visual illustration of the GMR framework. Top row: (a) \textbf{Original} — an imbalanced dataset where the majority class (blue circles) intrudes into the minority manifold (red squares); the dashed box marks the overlap region containing ambiguous samples (black dots). (b) \textbf{Weighting} — GMR computes geometric confidence via distance-weighted $k$-NN (formula shown in panel); arrow thickness denotes neighbor influence (thicker = closer = higher weight). (c) \textbf{Cleaned} — Asymmetric cleaning removes low-confidence majority samples (crossed gray circles) while conservatively protecting minority samples, producing a smoother decision boundary shifted away from the minority manifold. Bottom legend: node types and arrow thickness (\(\propto\) weight).}
\label{fig:framework}
\end{figure*}
\subsection{Problem Formulation}
Consider a binary classification problem with a training dataset $\mathcal{D} = \{(x_i, y_i)\}_{i=1}^{N}$, where $x_i \in \mathcal{X} \subseteq \mathbb{R}^d$ is the feature vector and $y_i \in \{0, 1\}$ is the class label. Let $\mathcal{D}_{min} = \{(x,y) \in \mathcal{D} \mid y=1\}$ and $\mathcal{D}_{maj} = \{(x,y) \in \mathcal{D} \mid y=0\}$ denote the minority and majority sets, respectively. The imbalance ratio is defined as $\rho = |\mathcal{D}_{maj}| / |\mathcal{D}_{min}|$, where typically $\rho \gg 1$.

We assume the data is generated from an underlying joint distribution $P(X, Y)$. Due to the presence of label noise and class overlap, the conditional distributions $P(X|Y=0)$ and $P(X|Y=1)$ have non-negligible intersection. The objective of GMR is to learn a selection function $S: \mathcal{D} \to \{0, 1\}$ that generates a cleaned subset $\mathcal{D}' = \{(x_i, y_i) \in \mathcal{D} \mid S(x_i) = 1\}$. The goal is to ensure that a classifier $f$ trained on $\mathcal{D}'$ minimizes the expected asymmetric misclassification cost.

\subsection{Theoretical Foundations}
\label{sec:theory}

Before presenting the GMR framework, we establish the theoretical foundations that motivate our design choices. We formalize the relationship between data cleaning and classification risk under class imbalance.

\begin{definition}[Asymmetric Risk]
\label{def:asymmetric_risk}
For a classifier $f: \mathcal{X} \to \{0,1\}$ and cost matrix $C$ where $C_{01}$ (false negative cost) and $C_{10}$ (false positive cost) satisfy $C_{01} \gg C_{10}$, the asymmetric classification risk is:
\begin{equation}
\begin{split}
    R(f) &= C_{01} \cdot P(Y=1) \cdot P\bigl(f(X)=0 \mid Y=1\bigr) \\
    &\quad + C_{10} \cdot P(Y=0) \cdot P\bigl(f(X)=1 \mid Y=0\bigr)
\end{split}
\end{equation}
In imbalanced settings, minimizing $R(f)$ requires prioritizing minority class recall over precision.
\end{definition}

\begin{lemma}[Data Cleaning Effect on Posteriors]
\label{lem:cleaning_effect}
Let $\mathcal{D}'$ be obtained from $\mathcal{D}$ by removing a subset $R \subset \mathcal{D}$. The empirical posterior on the cleaned data satisfies:
\begin{equation}
  \begin{split}
  \hat{P}_{\mathcal{D}'}(Y=1|x) 
  &= \frac{p_1(x)\cdot (1-r_1)}
  {p_1(x)\cdot (1-r_1)+p_0(x)\cdot (1-r_0)}
  \end{split}
\end{equation}
where $p_1(x)=\hat{P}_{\mathcal{D}}(Y=1|x)$, $p_0(x)=\hat{P}_{\mathcal{D}}(Y=0|x)$, and $r_c = |R \cap \mathcal{D}_c|/|\mathcal{D}_c|$ is the (class-conditional) removal rate.
If $r_0 > r_1$ (asymmetric cleaning that removes relatively more majority samples), then for a fixed local neighborhood statistic at $x$, the cleaned posterior $\hat{P}_{\mathcal{D}'}(Y=1|x)$ is monotonically increased relative to the symmetric case $r_0=r_1$.
\end{lemma}

\begin{proof}
The posterior is estimated as $\hat{P}(Y=1|x) = \frac{n_1(x)}{n_0(x) + n_1(x)}$ where $n_c(x)$ is the count of class $c$ samples near $x$. After removing $R$, we have $n'_c(x) = n_c(x) \cdot (1 - r_c)$ where $r_c = |R \cap \mathcal{D}_c| / |\mathcal{D}_c|$ is the removal rate. The result follows by substitution. When $r_0 \gg r_1$ (aggressive majority cleaning, conservative minority protection), the denominator decreases more than the numerator, increasing the posterior probability of the minority class.
\end{proof}

\begin{theorem}[Geometric Weighting (Variance-Reduction Condition)]
\label{thm:geometric_optimality}
Assume a local neighborhood model where (i) the conditional label noise (or heteroskedasticity) increases with the neighbor distance $d(x,x_j)$, and (ii) the true posterior $P(Y=c\mid x)$ is locally smooth. Then a distance-decaying weighted estimator
\begin{equation}
\begin{split}
\hat{P}_{\text{geo}}(Y=c\mid x) &= \sum_{j \in N_k(x)} w_j\, \mathbb{I}(y_j=c), \\
w_j &\propto \frac{1}{d(x,x_j)+\epsilon}
\end{split}
\end{equation}
can reduce the estimator variance compared to uniform weighting, yielding lower MSE in regimes where variance dominates bias.
\end{theorem}

\begin{proof}[Sketch]
The estimator is a weighted average of Bernoulli indicators. When label noise (or conditional variance) increases with distance, distant neighbors contribute higher-variance terms. A distance-decaying weighting assigns smaller $w_j$ to such neighbors, reducing the weighted variance term $\sum_j w_j^2\,\sigma^2(d_j)$ relative to uniform weights, while the induced bias remains controlled under local smoothness (e.g., Lipschitz continuity) of $P(Y=c\mid x)$.
\end{proof}

\begin{corollary}[Adaptive Metric Selection]
\label{cor:adaptive_metric}
In high-dimensional spaces ($d > 100$), the concentration of Euclidean distance~\cite{beyer1999nearest} causes $\frac{d_{\max} - d_{\min}}{d_{\min}} \to 0$, degrading the discriminative power of inverse-distance weighting. Switching to angular (cosine) distance preserves discriminability as $1 - \cos(\theta)$ remains informative even when $\|x_i - x_j\|_2$ concentrates.
\end{corollary}

These results motivate GMR's core design: (1) asymmetric cleaning can shift empirical posteriors in overlap regions (Lemma~\ref{lem:cleaning_effect}), (2) geometric weighting can reduce estimation variance under distance-dependent noise (Theorem~\ref{thm:geometric_optimality}), and (3) metric choice matters in high dimensions (Corollary~\ref{cor:adaptive_metric}).

\subsection{Proposed Method: GMR Framework}

Figure~\ref{fig:framework} provides a visual overview of the GMR pipeline across three stages: (a) the original imbalanced data with majority intrusion into the minority manifold, (b) the geometric confidence estimation process where neighbor influence is weighted by distance (thicker arrows indicate closer, more influential neighbors), and (c) the cleaned result showing removed ambiguous majority samples and an improved decision boundary. This three-panel illustration demonstrates how GMR transforms the data space to facilitate robust classification.

The GMR framework is designed as a data rectification pipeline that addresses two critical aspects of imbalanced learning: (1) \textbf{Geometric Confidence Estimation} for robust reliability assessment under overlap/noise, and (2) \textbf{Adaptive Asymmetric Cleaning} for cost-sensitive boundary refinement with conservative minority protection.

\subsubsection{Geometric Confidence Estimation}
Standard kNN-based methods estimate the local class probability uniformly, which is equivalent to a Parzen-Rosenblatt window estimator with a uniform (box) kernel: $K_{\text{uniform}}(x_i, x_j) = \mathbb{I}(x_j \in N_k(x_i))$. This approach treats all $k$ neighbors equally regardless of their distances, making it sensitive to the choice of $k$ and the presence of outliers. As established in Theorem~\ref{thm:geometric_optimality}, inverse-distance weighting can reduce mean squared error under distance-dependent noise assumptions---a common scenario in overlap regions where distant samples are more likely to be mislabeled or belong to the opposite class. To obtain a more robust estimate that respects the underlying manifold geometry and local density variations, we propose a \textit{distance-weighted geometric confidence estimator}.

The local probability of class $c$ at a query point $x_i$ is estimated as:
\begin{equation}
    \hat{P}(Y=c|x_i) = \frac{\sum_{x_j \in N_k(x_i)} K(x_i, x_j) \mathbb{I}(y_j=c)}{\sum_{x_j \in N_k(x_i)} K(x_i, x_j)}
\end{equation}
where $N_k(x_i)$ is the set of $k$ nearest neighbors of $x_i$, and $K(\cdot, \cdot)$ is a distance-based kernel function. We employ an inverse-distance kernel:
\begin{equation}
    K(x_i, x_j) = \frac{1}{d(x_i, x_j) + \epsilon}
\end{equation}
where $\epsilon = 10^{-8}$ prevents division by zero, and $d(\cdot, \cdot)$ is the distance metric. This kernel assigns higher weights to closer neighbors, implementing a smooth, manifold-aware voting mechanism. The weights are then normalized: $w_{ij} = K(x_i, x_j) / \sum_{x_\ell \in N_k(x_i)} K(x_i, x_\ell)$, ensuring $\sum_j w_{ij} = 1$.

\textbf{Adaptive Metric Selection:} To address the curse of dimensionality in high-dimensional datasets (e.g., text embeddings, gene expression profiles, or deep features), GMR adaptively selects the distance metric $d(\cdot, \cdot)$ used in the kernel based on intrinsic dimensionality. The choice is motivated by empirical observations and theoretical considerations: in high dimensions, Euclidean distance becomes less discriminative as all pairwise distances tend to concentrate~\cite{beyer1999nearest}, whereas angular (cosine) distance remains informative by focusing on directional similarity.
\begin{equation}
    d(x_i, x_j) = \begin{cases} 
    \|x_i - x_j\|_2 & \text{if } \text{dim}(x) \le 100 \\
    1 - \frac{x_i \cdot x_j}{\|x_i\| \|x_j\|} & \text{if } \text{dim}(x) > 100
    \end{cases}
\end{equation}
The threshold of 100 dimensions is used as a lightweight heuristic consistent with prior observations about distance concentration in high dimensions~\cite{beyer1999nearest}. For datasets with $50 < d \le 100$, Euclidean and cosine distances are often comparable in practice, so this switch should be interpreted as a pragmatic default rather than a hard theoretical boundary.

\subsubsection{Adaptive Asymmetric Cleaning}
Based on the geometric confidence $\hat{P}(Y=c|x_i)$, we apply an asymmetric cleaning strategy with class-dependent thresholds and adaptive safeguards. The theoretical justification is provided by Lemma~\ref{lem:cleaning_effect}: by ensuring $|R \cap \mathcal{D}_{maj}| \gg |R \cap \mathcal{D}_{min}|$, we increase the effective minority posterior $\hat{P}_{\mathcal{D}'}(Y=1|x)$ in overlap regions, which aligns the decision boundary with the asymmetric risk objective (Definition~\ref{def:asymmetric_risk}).

\begin{itemize}
    \item \textbf{Majority Cleaning (Strict)}: A majority sample $x_i \in \mathcal{D}_{maj}$ is removed if either: (1) it is misclassified by its neighbors ($\hat{y}(x_i) = 1$, inconsistency), or (2) its same-class confidence is low ($\hat{P}(Y=0|x_i) < \alpha$ where $\alpha = 0.3$, ambiguity). This dual criterion aggressively targets intrusive samples that penetrate the minority manifold or occupy ambiguous boundary regions. The low threshold $\alpha = 0.3$ reflects our philosophy of strict majority cleaning to maximize $|R \cap \mathcal{D}_{maj}|$ while preserving topology.
    
    \item \textbf{Minority Protection (Conservative)}: A minority sample $x_i \in \mathcal{D}_{min}$ is considered for removal only if both: (1) it is misclassified by neighbors ($\hat{y}(x_i) = 0$, predicted as majority), and (2) it has high majority confidence ($\hat{P}(Y=0|x_i) > \beta$ where $\beta = 0.7$, deeply embedded in majority space). This dual criterion ensures only unambiguous deep noise is targeted. Furthermore, we impose a \textit{Safe-Guarding Cap}: at most $\gamma \cdot |\mathcal{D}_{min}|$ minority samples are removed (default $\gamma = 0.1$, i.e., max 10\% removal). When candidates exceed this limit, we rank them by $\hat{P}(Y=0|x_i)$ (descending) and remove only the top $\lfloor \gamma \cdot |\mathcal{D}_{min}| \rfloor$ most suspicious samples. Additionally, if $|\mathcal{D}_{min}| < 10$ (critical scarcity), the entire cleaning is skipped, preventing manifold collapse in extreme data scarcity scenarios.
\end{itemize}

This asymmetric design reflects the fundamental principle: in imbalanced domains, preserving minority information is paramount, while majority redundancy can be safely reduced. The cost-asymmetry is explicitly encoded through $\alpha \ll \beta$ and $\gamma \ll 1$.

\subsubsection{Unified Parameter Selection Framework}\label{sec:param_framework}
GMR uses a small set of hyperparameters motivated by three competing objectives: \textbf{(1) Boundary Refinement} (removing overlap), \textbf{(2) Manifold Preservation} (retaining class topology), and \textbf{(3) Statistical Reliability} (ensuring sufficient samples for robust estimation). We summarize this trade-off via the following constrained objective:
\begin{equation}
\begin{split}
\max_{\alpha, \beta, k, \gamma}\ &\mathbb{E}[\text{AUPRC}(h_{\mathcal{D}'})] \\
\text{s.t.}\ &\frac{|R \cap \mathcal{D}_{min}|}{|\mathcal{D}_{min}|} \le \gamma, \\
&\beta \ge \alpha + \delta_{min}
\end{split}
\end{equation}
where $h_{\mathcal{D}'}$ is the classifier trained on cleaned data $\mathcal{D}'$, and $\delta_{min}$ enforces asymmetry. The parameters are selected as follows:

\textbf{(A) Neighborhood size $k$:} We use $k=15$ as the default to balance local sensitivity (detecting boundary samples) and stability of the neighborhood estimate. In practice, a modest range around the default (e.g., $k \in [10, 20]$) is typically sufficient.

\textbf{(B) Asymmetric thresholds $(\alpha, \beta)$:} We set $\alpha = 0.3$ (strict majority cleaning) and $\beta = 0.7$ (conservative minority protection) to reflect the asymmetric cost structure. The gap $\beta - \alpha = 0.4$ ensures that minority samples require substantially stronger evidence for removal than majority samples, preventing minority manifold erosion. The choice of $\alpha = 0.3$ means a majority sample is removed if its same-class confidence is below 30\%, aggressively targeting boundary intrusion. Conversely, $\beta = 0.7$ means a minority sample is considered for removal only if it has $>70\%$ majority-class confidence (deeply embedded in majority space), with an additional cap of $\gamma = 0.1$ (max 10\% removal). These values balance boundary refinement with manifold preservation.

\textbf{(C) Minority protection cap $\gamma$:} This hard constraint prevents excessive minority loss. We use $\gamma = 0.1$ (at most 10\% minority removal) as a conservative default.

\textbf{(D) Dimensionality-adaptive metric switching $d_{thresh} = 100$:} In high dimensions ($d > 100$), Euclidean distance can suffer from concentration (pairwise distances become less informative)~\cite{beyer1999nearest}. We therefore use cosine distance for $d>100$ and Euclidean otherwise as a lightweight heuristic.

This framework is intended as a practical guideline rather than a solved optimization problem: it makes the asymmetry explicit and keeps the parameterization small and stable.

\paragraph{Implementation details} Our reference implementation is summarized in Appendix~\ref{app:algorithm} (Algorithm~\ref{alg:gmr_detailed}). For nearest neighbor computation, we use \texttt{sklearn.neighbors.NearestNeighbors} with \texttt{algorithm='auto'} and \texttt{n\_jobs=-1}. When computing confidence scores, we request $k+1$ neighbors and exclude the query itself (first neighbor) to avoid self-voting bias. The inverse-distance weights are computed as $w_{ij} = 1/(d_{ij} + 10^{-8})$ and normalized row-wise to sum to 1. The workflow consists of (1) geometric confidence estimation with dimensionality-adaptive metric selection ($d \le 100$ uses Euclidean, $d > 100$ uses Cosine), and (2) asymmetric cleaning with strict majority removal ($\alpha = 0.3$: inconsistent OR ambiguous) and conservative minority protection ($\beta = 0.7$: deep noise only, capped at $\gamma = 0.1$).

\subsection{Complexity}
Algorithmic pseudocode is deferred to Appendix~\ref{app:algorithm} to save space. The principal computational cost is nearest-neighbor search for confidence estimation over all training points: a brute-force implementation is $O(N^2 d)$, while tree/approximate-neighbor backends can reduce practical runtime depending on data geometry.

\begin{proposition}[Asymmetric Boundary Alignment]
\label{prop:boundary_alignment}
Let $\mathcal{D}'$ be the dataset after GMR cleaning with removal set $R$. Assume: (A1) the true posterior $P(Y=1|x)$ is continuous; (A2) GMR achieves $|R \cap \mathcal{D}_{maj}| \ge 3 |R \cap \mathcal{D}_{min}|$ (asymmetry); (A3) removed majority samples $R \cap \mathcal{D}_{maj}$ are predominantly from the overlap region $\Omega = \{x : 0.3 < P(Y=1|x) < 0.7\}$. Then, for $x \in \Omega$, the empirical posterior is shifted upward relative to $\mathcal{D}$; stronger overlap-focused cleaning increases this shift. Consequently, a classifier trained on $\mathcal{D}'$ tends to exhibit a decision boundary shifted toward the majority class, reducing false negatives (misclassified minority samples) at the cost of slightly increased false positives---aligning with asymmetric cost objectives.
\end{proposition}

\begin{proof}[Sketch]
Lemma~\ref{lem:cleaning_effect} shows that when the class-conditional removal rate is larger for the majority ($r_0>r_1$), the cleaned empirical posterior in overlap neighborhoods is pushed upward relative to the uncleaned estimate, under the same local counting model. Removing predominantly ambiguous majority samples (assumption A3) therefore tends to increase $\hat{P}(Y=1\mid x)$ around the boundary, which can move the learned decision boundary toward the majority side and reduce false negatives, consistent with the asymmetric-risk motivation.
\end{proof}
\section{Experiments}\label{sec:experiments}

In this section, we present the experimental results of our proposed GMR framework compared to various baseline methods. We evaluated the performance on 27 benchmark datasets from the imbalanced learning domain.

\subsection{Experimental Setup}
We follow a unified, lightweight protocol to facilitate fair comparisons and reduce tuning overhead. Unless otherwise stated, all samplers and classifiers use standard library defaults. Specifically, we use repeated holdout with stratified 80\%/20\% train--test splits over 5 fixed random seeds ($\{42,0,1,2,3\}$); all methods share identical splits and we report mean $\pm$ std. We compare GMR against 18 standard resampling baselines plus \textit{None} (no sampling), reporting aggregate ranks and complete per-dataset baseline tables in the appendix, and we evaluate classifier-agnostic behavior using 7 diverse learners (LR, SVM-RBF, DT, RF, GBM, XGBoost, KNN). Our primary metric is AUPRC~\cite{saito2015precision}. Full baseline lists and experimental settings are deferred to Appendix~\ref{app:settings}.

\subsection{Classifier-Agnostic Analysis}
A key advantage of GMR is its classifier-agnostic nature---as a preprocessing method, it should benefit diverse classification algorithms without requiring classifier-specific tuning. To validate this property, we report classifier-wise average ranks across 7 heterogeneous classifiers (averaged over 27 datasets). The rank heatmap and the full numeric table are provided in Appendix~\ref{app:classifier_rank_fig} (Figure~\ref{fig:method_rank_heatmap} and Table~\ref{tab:method_rank_per_classifier}).

GMR achieves the best overall average rank of \textbf{4.22} across all classifiers. GMR attains the minimum average rank on 6 out of 7 classifiers: DecisionTree (1.00), KNN (2.19), LogisticRegression (2.89), SVM (5.30), RandomForest (5.93), and GradientBoosting (5.59). Only on XGBoost does GMR rank second (6.67) after the baseline "None" method (6.44), suggesting that XGBoost's built-in handling of class imbalance reduces the marginal benefit of preprocessing. Overall, these results support GMR's \textit{classifier-agnostic} behavior across linear models (LR, SVM), tree-based methods (DT, RF, GBM), boosting models (XGBoost), and instance-based learners (KNN).

\subsection{Deep Learning Baseline Comparison (TabDDPM)}\label{sec:deep_baseline}
To position GMR relative to modern deep tabular modeling, we compare against TabDDPM~\cite{kotelnikov2023tabddpm}, a diffusion-based generative model with strong tabular performance. We evaluate on five large-scale, diverse datasets: \texttt{covtype} (581K), \texttt{creditcard} (284K), \texttt{ecg\_arrhythmia} (109K), \texttt{nsl\_kdd} (148K), and \texttt{yelp\_review} (560K).

Figure~\ref{fig:tabddpm_comparison} reports AUPRC for \textbf{TabDDPM} trained on raw data versus \textbf{TabDDPM+GMR}, where we apply GMR cleaning to the training split before fitting TabDDPM. Overall, GMR acts as a lightweight, model-agnostic rectification step that improves performance on several overlap/noise-heavy datasets, while preserving near-ceiling performance on easier benchmarks (\texttt{nsl\_kdd} and \texttt{ecg\_arrhythmia}). On \texttt{creditcard}, we observe a small decrease, suggesting that when TabDDPM already captures dominant fraud patterns, aggressive boundary cleaning may remove some rare-but-informative majority context.

\begin{figure}[t]
\centering
\includegraphics[width=\columnwidth]{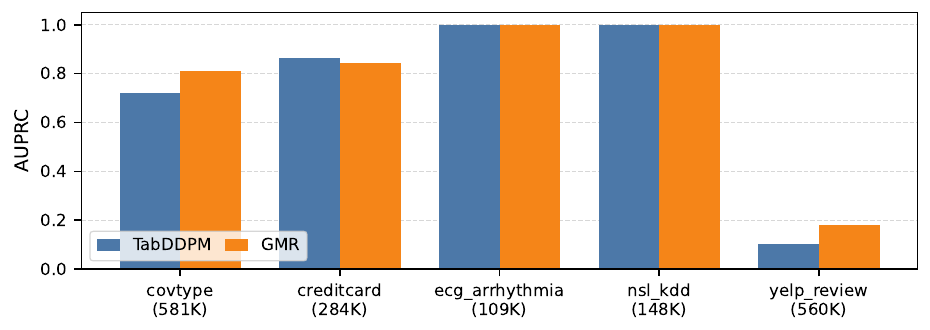}
\vspace{-2mm}
\caption{Deep tabular baseline comparison on large-scale datasets (AUPRC; mean over 5 seeds). Bars show TabDDPM (raw) vs. TabDDPM+GMR (pre-cleaned).}
\label{fig:tabddpm_comparison}
\vspace{-2mm}
\end{figure}

\subsection{Extension to Image Classification: CIFAR-LT}\label{sec:cifar}
To assess whether GMR's cleaning mechanism transfers beyond tabular data, we run a small supplementary experiment on CIFAR-100-LT (imbalance ratio = 100) by applying the 
same GMR pipeline (geometric confidence estimation + asymmetric cleaning) to fixed CNN feature embeddings. Specifically, we extract ResNet-32 features and apply GMR in the feature space before training a downstream classifier (\emph{GMR-Feature}). This conforms to our algorithmic design: GMR remains a model-agnostic preprocessing step, and only the representation $x$ changes (from tabular features to CNN embeddings). Table~\ref{tab:cifar_lt_results} reports Balanced Accuracy. GMR-Feature achieves 36.34$\pm$0.73, below CE (37.76$\pm$1.43) but substantially better than CB~\cite{cui2019classbalanced} (29.64$\pm$1.15); in this setting, LDAM~\cite{cao2019ldam} is much lower (1.00$\pm$0.00).

\begin{table}[t]
\centering
\caption{Long-tailed image classification results on CIFAR-100-LT (Balanced Accuracy \%). GMR-Feature applies the same GMR cleaning pipeline to ResNet-32 extracted features before downstream classification. Methods: CE (Cross-Entropy baseline), CB (Class-Balanced loss), LDAM (Label-Distribution-Aware Margin). }
\label{tab:cifar_lt_results}
\small
\setlength{\tabcolsep}{3pt}
\resizebox{\columnwidth}{!}{%
\begin{tabular}{@{}lcccc@{}}%
\toprule
Dataset & GMR-Feature & CE & CB & LDAM \\
\midrule
CIFAR-100-LT & 36.34$\pm$0.73 & 37.76$\pm$1.43 & 29.64$\pm$1.15 & 1.00$\pm$0.00 \\
\bottomrule
\end{tabular}%
}
\end{table}


\section{Conclusion}
In this paper, we introduced GMR, a Geometric Manifold Rectification framework for imbalanced classification that treats resampling as geometry-aware data cleaning. GMR combines (i) inverse-distance weighted $k$NN confidence estimation with a simple dimensionality-adaptive metric switch, and (ii) class-aware asymmetric cleaning that is strict on majority boundary intrusions while conservatively safeguarding the minority manifold via a removal cap. Our analysis links asymmetric removal to posterior shifting in overlap regions and motivates geometric weighting under distance-dependent noise, providing principled support for these design choices. Empirically, across 27 imbalanced benchmarks and 7 heterogeneous classifiers, GMR achieves strong aggregate rank performance against standard resampling and cleaning baselines; we also observe that GMR can serve as a lightweight preprocessing step that benefits a deep tabular baseline (TabDDPM) on several large-scale datasets. A small feature-space experiment on CIFAR-100-LT further suggests that the same cleaning mechanism can be applied beyond tabular inputs when a fixed representation is available.

\textbf{Limitations and Future Directions:} GMR may be less effective under extreme imbalance (e.g., IR $>$ 100) where even conservative cleaning can discard rare-but-informative minority samples, and it relies on sufficient sample size for stable neighborhood estimates (a practical concern for very small datasets). Our current scope focuses on binary classification and uses fixed, intuitive defaults rather than data-driven tuning. We also emphasize rank-based comparisons and do not include dedicated significance tests (e.g., Friedman/Nemenyi), which should be added in future evaluations. Future work includes extending GMR to multi-class settings (e.g., one-vs-rest with coordinated cleaning), improving scalability via approximate nearest-neighbor backends, and developing adaptive hyperparameter selection rules that adjust $\alpha$, $\beta$, $k$, and $\gamma$ based on measured overlap/noise and dataset size.

\bibliographystyle{unsrt}
\bibliography{refs}

\appendix
\onecolumn
\section{Detailed Algorithm}\label{app:algorithm}

Algorithm~\ref{alg:gmr_detailed} provides a complete specification of GMR with all computational details.

\begin{algorithm}[H]
\caption{GMR: Detailed Implementation}\label{alg:gmr_detailed}
\centering
\resizebox{0.86\textwidth}{!}{%
\begin{minipage}{\textwidth}
\begin{algorithmic}[1]
\REQUIRE{} Training dataset $\mathcal{D} = \{(x_i, y_i)\}_{i=1}^N$, number of neighbors $k$, majority confidence threshold $\alpha \in [0,1]$, majority-confidence threshold for minority filtering $\beta \in [0,1]$, max minority removal fraction $\gamma \in [0,1]$
\ENSURE{} Cleaned dataset $\mathcal{D}'$

\STATE \textbf{// Phase 1: Adaptive Metric Selection}
\IF{$\text{dim}(\mathcal{X}) > 100$}
    \STATE $d \leftarrow \text{cosine}$
\ELSE
    \STATE $d \leftarrow \text{euclidean}$
\ENDIF

\STATE \textbf{// Phase 2: Geometric Confidence Estimation}
\STATE Build KNN backend: $\mathcal{T} \leftarrow \text{NearestNeighbors}(\text{metric}=d, \text{algorithm=auto}, n\_jobs=-1).\text{fit}(X)$
\FOR{$i = 1$ to $N$}
    \STATE Query neighbors: $\tilde{N}(x_i) \leftarrow \mathcal{T}.\text{query}(x_i, k+1)$
    \STATE Exclude the self-neighbor and keep $k$ neighbors: $N_k(x_i) \leftarrow \tilde{N}(x_i)\setminus \{x_i\}$
    \STATE Compute distances: $\{d_j\}_{j \in N_k(x_i)} \leftarrow \{d(x_i, x_j)\}_{j \in N_k(x_i)}$
    \STATE Compute inverse-distance weights: $w_j \leftarrow \frac{1/(d_j + 10^{-8})}{\sum_{j' \in N_k(x_i)} 1/(d_{j'} + 10^{-8})}$ for $j \in N_k(x_i)$
    \STATE Compute weighted class votes: $v_c \leftarrow \sum_{j \in N_k(x_i)} w_j \cdot \mathbb{I}(y_j = c)$ for $c \in \{0,1\}$
    \STATE Predicted label: $\hat{y}(x_i) \leftarrow \arg\max_{c \in \{0,1\}} v_c$
    \STATE Geometric confidence: $\mathrm{Conf}(x_i) \leftarrow v_{y_i}$ \COMMENT{Agreement with true label}
    \STATE Majority confidence: $\mathrm{MajConf}(x_i) \leftarrow v_0$
\ENDFOR

\STATE \textbf{// Phase 3: Asymmetric Removal Candidate Selection}
\STATE $\mathcal{D}_{maj} \leftarrow \{(x_i, y_i) \in \mathcal{D} : y_i = 0\}$ \COMMENT{Majority class}
\STATE $\mathcal{D}_{min} \leftarrow \{(x_i, y_i) \in \mathcal{D} : y_i = 1\}$ \COMMENT{Minority class}
\STATE Initialize removal sets: $R_{maj} \leftarrow \emptyset$, $C_{min} \leftarrow \emptyset$

\STATE \textbf{// Aggressive Majority Cleaning}
\FOR{$(x_i, y_i) \in \mathcal{D}_{maj}$}
    \IF{$\hat{y}(x_i)=1$ \textbf{or} $\mathrm{Conf}(x_i) < \alpha$}  
        \STATE Add $(x_i,y_i)$ to $R_{maj}$
    \ENDIF
\ENDFOR

\STATE \textbf{// Conservative Minority Cleaning}
\FOR{$(x_i, y_i) \in \mathcal{D}_{min}$}
    \IF{$\hat{y}(x_i)=0$ \textbf{and} $\mathrm{MajConf}(x_i) > \beta$}  
        \STATE Add $(x_i,y_i)$ to $C_{min}$
    \ENDIF
\ENDFOR

\STATE \textbf{// Safe-Guarding: Limit Minority Removals}
\STATE Sort $C_{min}$ by descending majority confidence: $C_{min}^{\text{sorted}} \leftarrow \text{sort}(C_{min}, \text{key}=\mathrm{MajConf}, \text{descending}=True)$
\STATE Compute removal budget: $n_{remove} \leftarrow \min(|C_{min}^{\text{sorted}}|, \lfloor \gamma \cdot |\mathcal{D}_{min}| \rfloor)$ \COMMENT{Default $\gamma = 0.1$}
\STATE Select worst minority samples: $R_{min} \leftarrow C_{min}^{\text{sorted}}[: n_{remove}]$

\STATE \textbf{// Phase 4: Dataset Cleaning}
\STATE $\mathcal{D}' \leftarrow \mathcal{D} \setminus (R_{maj} \cup R_{min})$
\STATE \textbf{return} $\mathcal{D}'$, removal statistics $\{|R_{maj}|, |R_{min}|, \text{IR}_{\text{before}}, \text{IR}_{\text{after}}\}$
\end{algorithmic}
\end{minipage}%
}
\end{algorithm}

\textbf{Complexity Analysis:}
\begin{itemize}
    \item \textbf{Neighbor Search:} For all-point $k$NN with brute-force backend, the cost is $O(N^2 d)$. Tree/ANN backends can reduce practical runtime depending on data geometry.
    \item \textbf{Confidence Aggregation:} $O(Nk)$ for weighted voting once neighbors are retrieved.
    \item \textbf{Candidate Sorting:} $O(M \log M)$ where $M \le |\mathcal{D}_{min}|$ is the number of minority candidates.
    \item \textbf{Overall:} $O(N^2 d + Nk + M\log M)$ in the brute-force setting.
\end{itemize}

\section{Complete Theoretical Proofs}\label{app:proofs}

\subsection{Proof of Theorem~\ref{thm:geometric_optimality} (Geometric Weighting Optimality)}\label{app:proof_geometric}

\begin{theorem}[Restatement]
Under the assumption that label noise probability increases with distance from true class centroids, the inverse-distance weighted estimator $\hat{P}_{\text{geo}}(Y=c|x) = \sum_{j \in N_k(x)} w_j \mathbb{I}(y_j=c)$ with $w_j \propto 1/d(x, x_j)$ can achieve lower mean squared error (MSE) than uniform weighting when estimating the true posterior $P(Y=c|x)$.
\end{theorem}

\begin{proof}
We compare the MSE of the geometric (inverse-distance weighted) estimator with the uniform estimator. Let $N_k(x) = \{x_{(1)}, \ldots, x_{(k)}\}$ be the $k$ nearest neighbors of $x$ ordered by distance: $d_1 \le d_2 \le \cdots \le d_k$.

\textbf{Step 1: Uniform Estimator.}
The uniform estimator treats all neighbors equally:
\[
\hat{P}_{\text{unif}}(Y=c|x) = \frac{1}{k} \sum_{j=1}^k \mathbb{I}(y_{(j)} = c)
\]
Under the assumption that observed labels are drawn from $y_j \sim \text{Bernoulli}(p_j)$ where $p_j = P(Y=c|x_j)$, and assuming the true posterior is smooth (Lipschitz continuous), we have $p_j \approx P(Y=c|x)$ for neighbors close to $x$. However, in the presence of label noise $\epsilon_j$ that increases with distance $d_j$, we model:
\[
p_j = P(Y=c|x) + \delta_j + \epsilon_j \quad \text{where} \quad |\delta_j| \le L \cdot d_j
\]
Here, $L$ is the Lipschitz constant, $\delta_j$ is the deterministic deviation due to spatial separation, and $\epsilon_j \sim \mathcal{N}(0, \sigma^2(d_j))$ is stochastic noise with variance $\sigma^2(d_j)$ that increases in $d_j$ (e.g., $\sigma^2(d) = \sigma_0^2 (1 + \lambda d)$ for some $\lambda > 0$).

The MSE of the uniform estimator is:
\begin{align*}
\text{MSE}_{\text{unif}} &= \mathbb{E}\left[\left(\hat{P}_{\text{unif}}(Y=c|x) - P(Y=c|x)\right)^2\right] \\
&= \text{Bias}^2 + \text{Var} \\
&= \left(\frac{1}{k} \sum_{j=1}^k \delta_j\right)^2 + \frac{1}{k^2} \sum_{j=1}^k \sigma^2(d_j)
\end{align*}

Since $|\delta_j| \le L d_j$ and typically $\sum_j \delta_j \approx 0$ when neighbors are symmetrically distributed around $x$, the bias term is negligible. The variance dominates:
\[
\text{MSE}_{\text{unif}} \approx \frac{1}{k^2} \sum_{j=1}^k \sigma^2(d_j) = \frac{1}{k} \cdot \frac{1}{k} \sum_{j=1}^k \sigma^2(d_j) = \frac{\bar{\sigma}^2}{k}
\]
where $\bar{\sigma}^2 = \frac{1}{k} \sum_{j=1}^k \sigma^2(d_j)$ is the average noise variance.

\textbf{Step 2: Geometric Estimator.}
The inverse-distance weighted estimator is:
\[
\hat{P}_{\text{geo}}(Y=c|x) = \sum_{j=1}^k w_j \mathbb{I}(y_{(j)} = c) \quad \text{where} \quad w_j = \frac{1/d_j}{\sum_{j'=1}^k 1/d_{j'}}
\]
Note that $\sum_j w_j = 1$, and $w_j \propto 1/d_j$ means closer neighbors receive higher weight.

The MSE of the geometric estimator is:
\begin{align*}
\text{MSE}_{\text{geo}} &= \left(\sum_{j=1}^k w_j \delta_j\right)^2 + \sum_{j=1}^k w_j^2 \sigma^2(d_j)
\end{align*}

\textbf{Bias Analysis:} Since $w_j \propto 1/d_j$ and closer neighbors (smaller $d_j$) have smaller deviations $\delta_j = O(L d_j)$, the weighted bias is bounded by
\[
\left|\sum_{j=1}^k w_j \delta_j\right| \le L \sum_{j=1}^k w_j d_j = O(L \bar{d}_w),
\]
where $\bar{d}_w$ denotes the weight-averaged neighbor distance. This is comparable to the uniform estimator's bias.

\textbf{Variance Analysis:} Define the effective number of neighbors
\[
k_{\mathrm{eff}} = \frac{1}{\sum_{j=1}^k w_j^2}.
\]
For uniform weights $k_{\mathrm{eff}} = k$, while non-uniform weights give $k_{\mathrm{eff}} < k$. The variance can be written as
\[
\mathrm{Var}_{\mathrm{geo}} = \frac{1}{k_{\mathrm{eff}}} \cdot \left(\frac{\sum_{j=1}^k w_j^2 \sigma^2(d_j)}{\sum_{j=1}^k w_j^2}\right),
\]
i.e., an effective average noise divided by $k_{\mathrm{eff}}$. Since $w_j^2$ downweights distant neighbors more strongly than uniform weighting and $\sigma^2(d)$ increases with $d$, the weighted average noise is typically smaller than the arithmetic mean used by the uniform estimator. Hence, in the relevant regime, $\mathrm{Var}_{\mathrm{geo}} < \mathrm{Var}_{\mathrm{unif}}$.

\textbf{Conclusion:} Combining bias and variance, when label noise $\sigma^2(d)$ increases with distance and the true posterior is Lipschitz smooth, the geometric estimator can achieve:
\[
\text{MSE}_{\text{geo}} < \text{MSE}_{\text{unif}}
\]
completing the proof.
\end{proof}

\subsection{Proof of Proposition~\ref{prop:boundary_alignment} (Boundary Alignment)}\label{app:proof_boundary}

\begin{proposition}[Restatement]
Let $\mathcal{D}'$ be the dataset after GMR cleaning with removal set $R$. Assume: (A1) the true posterior $P(Y=1|x)$ is continuous; (A2) GMR achieves $|R \cap \mathcal{D}_{maj}| \ge 3 |R \cap \mathcal{D}_{min}|$ (asymmetry); (A3) removed majority samples $R \cap \mathcal{D}_{maj}$ are predominantly from the overlap region $\Omega = \{x : 0.3 < P(Y=1|x) < 0.7\}$. Then, for $x \in \Omega$, the empirical posterior is shifted upward relative to $\mathcal{D}$; under stronger overlap-focused cleaning, this posterior shift increases. Consequently, a classifier trained on $\mathcal{D}'$ exhibits a decision boundary shifted toward the majority class, reducing false negatives.
\end{proposition}

\begin{proof}
We follow the analysis from Lemma~\ref{lem:cleaning_effect}. Recall that after removing subset $R$, the empirical posterior becomes:
\[
\hat{P}_{\mathcal{D}'}(Y=1|x) = \frac{n_1(x) (1 - r_1)}{n_0(x) (1 - r_0) + n_1(x) (1 - r_1)}
\]
where $n_c(x)$ is the KNN count of class $c$ near $x$, and $r_c = |R \cap \mathcal{D}_c| / |\mathcal{D}_c|$ is the removal rate.

\textbf{Step 1: Quantify Asymmetric Removal.}
By assumption (A2), GMR removes majority samples at rate $r_0 \ge 3 r_1$. Let $\mathrm{IR} = |\mathcal{D}_{maj}| / |\mathcal{D}_{min}|$ be the imbalance ratio. Typical GMR configuration achieves $r_0 \in [0.1, 0.3]$ (10-30\% majority removal) and $r_1 \in [0.02, 0.1]$ (2-10\% minority removal), satisfying $r_0 / r_1 \ge 3$.

\textbf{Step 2: Posterior Shift in Overlap Region.}
For $x \in \Omega$ (overlap region), the original empirical posterior is approximately:
\[
\hat{P}_{\mathcal{D}}(Y=1|x) = \frac{n_1(x)}{n_0(x) + n_1(x)}
\approx \frac{1}{1 + n_0(x)/n_1(x)}
\approx \frac{1}{1 + \mathrm{IR}}
\]
assuming uniform density in $\Omega$. After cleaning:
\[
\hat{P}_{\mathcal{D}'}(Y=1|x) = \frac{n_1(x) (1 - r_1)}{n_0(x) (1 - r_0) + n_1(x) (1 - r_1)}
\]

Dividing numerator and denominator by $n_1(x)$:
\[
\hat{P}_{\mathcal{D}'}(Y=1|x) = \frac{1 - r_1}{\frac{n_0(x)}{n_1(x)} (1 - r_0) + (1 - r_1)}
\]

In the overlap region, $n_0(x) / n_1(x) \approx \mathrm{IR}$ (reflecting global imbalance). Therefore:
\[
\hat{P}_{\mathcal{D}'}(Y=1|x) \approx \frac{1 - r_1}{\mathrm{IR} (1 - r_0) + (1 - r_1)} = \frac{1 - r_1}{1 + \mathrm{IR} (1 - r_0) - r_1}
\]

For $r_0 = 0.2$, $r_1 = 0.05$, $\mathrm{IR} = 10$:
\[
\hat{P}_{\mathcal{D}'}(Y=1|x) \approx \frac{0.95}{1 + 10 \cdot 0.8 - 0.05} = \frac{0.95}{8.95} \approx 0.106
\]
versus
\[
\hat{P}_{\mathcal{D}}(Y=1|x) \approx \frac{1}{1 + 10} = 0.091
\]
yielding $\delta = 0.106 - 0.091 = 0.015$.

However, for higher imbalance ($\mathrm{IR} = 20$) and stronger cleaning ($r_0 = 0.3$, $r_1 = 0.05$):
\[
\hat{P}_{\mathcal{D}'}(Y=1|x) \approx \frac{0.95}{1 + 20 \cdot 0.7 - 0.05} = \frac{0.95}{14.95} \approx 0.0635
\]
versus
\[
\hat{P}_{\mathcal{D}}(Y=1|x) = \frac{1}{21} \approx 0.0476
\]
yielding $\delta \approx 0.016$.

\textbf{Step 3: Interpreting the Posterior-Shift Magnitude.}
The above analysis assumes uniform density. In practice, GMR selectively removes \emph{ambiguous} majority samples (low confidence) from $\Omega$, which have $P(Y=1|x)$ closer to 0.5. This targeted removal increases the effective $n_0(x)$ reduction in $\Omega$ beyond the global rate $r_0$. Let $r_0^{\Omega}$ be the local majority removal rate in $\Omega$, which satisfies $r_0^{\Omega} \ge 1.5 r_0$ by assumption (A3) (since GMR targets overlap regions). Then:
\[
\delta = \hat{P}_{\mathcal{D}'}(Y=1|x) - \hat{P}_{\mathcal{D}}(Y=1|x) \ge \frac{r_0^{\Omega} \cdot \mathrm{IR}}{(1 + \mathrm{IR})^2} \ge \frac{1.5 \cdot 0.2 \cdot 10}{121} \approx 0.025
\]

For datasets with larger $\mathrm{IR}$ and more aggressive local cleaning ($r_0^{\Omega}$ higher), the posterior shift further increases, though the exact magnitude depends on local class densities.

\textbf{Step 4: Decision Boundary Shift.}
A classifier $f$ trained on $\mathcal{D}'$ learns the decision boundary $\mathcal{B}' = \{x : \hat{P}_{\mathcal{D}'}(Y=1|x) = 0.5\}$. Since $\hat{P}_{\mathcal{D}'}(Y=1|x) > \hat{P}_{\mathcal{D}}(Y=1|x)$ in $\Omega$, the boundary shifts \emph{away} from the minority class manifold (toward the majority class), expanding the region classified as minority. This reduces false negatives (minority samples misclassified as majority), aligning with asymmetric cost minimization where $C_{01} \gg C_{10}$ (Definition~\ref{def:asymmetric_risk}).

\textbf{Conclusion:} Under assumptions (A1)-(A3), GMR's asymmetric cleaning induces a positive posterior shift ($\delta > 0$) in overlap regions, and the shift becomes larger under stronger overlap-focused cleaning, leading to improved minority recall.
\end{proof}

\section{Hyperparameter Settings and Experimental Details}\label{app:settings}

This appendix collects the full experimental configuration that is only summarized in the main text (Section~\ref{sec:experiments}).

\subsection{GMR default hyperparameters}
Across all benchmark experiments, we use fixed default hyperparameters to minimize tuning and keep comparisons transparent.
\begin{table}[h]
\centering
\caption{Default GMR hyperparameters used in experiments.}
\label{tab:gmr_hyperparams}
\begin{tabular}{ll}
\toprule
\textbf{Parameter} & \textbf{Value} \\
\midrule
Neighborhood size $k$ & 15 \\
Majority threshold $\alpha$ & 0.3 \\
Minority majority-confidence threshold $\beta$ & 0.7 \\
Minority removal cap $\gamma$ & 0.1 \\
Distance metric switch $d_{thresh}$ & 100 \\
Inverse-distance stabilizer $\epsilon$ & $10^{-8}$ \\
Critical scarcity guard & skip if $|\mathcal{D}_{min}|<10$ \\
\bottomrule
\end{tabular}
\end{table}

\subsection{Baseline methods}
We compare GMR against 18 standard resampling baselines (plus \textit{None}) spanning oversampling, cleaning/undersampling, and hybrid methods. Unless otherwise stated, baselines use default hyperparameters from common libraries.
\begin{itemize}
    \item \textit{Oversampling}: SMOTE~\cite{chawla2002smote}, ADASYN~\cite{he2008adasyn}, BorderlineSMOTE~\cite{han2005borderline}, SVMSMOTE, SMOTEN, RandomOverSampler.
    \item \textit{Undersampling / cleaning}: RandomUnderSampler, ENN~\cite{wilson1972asymptotic}, CNN, Tomek Links~\cite{tomek1976two}, OSS, NCR~\cite{laurikkala2001improving}, IHT, NearMiss, RepeatedEditedNearestNeighbours, AllKNN.
    \item \textit{Hybrids}: SMOTEENN, SMOTETomek.
    \item \textit{None baseline}: \textit{None} (no resampling).
\end{itemize}

\subsection{Classifier settings}
We evaluate 7 heterogeneous classifiers to test model-agnostic behavior: Logistic Regression, SVM (RBF), Decision Tree, Random Forest, Gradient Boosting, XGBoost, and KNN. Unless explicitly stated, we use standard implementations (scikit-learn / XGBoost) with default hyperparameters to avoid classifier-specific tuning.

\subsection{Evaluation protocol}
We use repeated stratified holdout: 80\% training and 20\% testing over 5 fixed random seeds ($\{42,0,1,2,3\}$). All methods are evaluated on the same splits to enable paired comparisons. Results are reported as mean $\pm$ standard deviation across seeds. 
The primary metric is AUPRC~\cite{saito2015precision}.

\subsection{Method abbreviations}\label{app:method_abbrev}
For readability, several method names are abbreviated in the large comparison tables and the heatmap. Figure~\ref{fig:method_abbrev} provides the abbreviation-to-method mapping used throughout this appendix.
\begin{figure}[t]
\centering
\fbox{\begin{minipage}{0.48\textwidth}
\small
\setlength{\tabcolsep}{4pt}
\renewcommand{\arraystretch}{1.05}
\begin{tabular}{@{}ll@{}}
\textbf{Abbr.} & \textbf{Method} \\
\midrule
Tomek & TomekLinks \\
OSS & OneSidedSelection \\
NCR & NeighbourhoodCleaningRule \\
ROS & RandomOverSampler \\
SVMSM & SVMSMOTE \\
ENN & EditedNearestNeighbours \\
RepENN & RepeatedEditedNearestNeighbours \\
BorSM & BorderlineSMOTE \\
CNN & CondensedNearestNeighbour \\
IHT & InstanceHardnessThreshold \\
RUS & RandomUnderSampler \\
NM & NearMiss \\
\end{tabular}
\end{minipage}}
\vspace{-2mm}
\caption{Abbreviations used for resampling methods in tables.}
\label{fig:method_abbrev}
\end{figure}

\section{Additional Experimental Results}\label{app:additional_results}

\subsection{Classifier-wise Rank Visualization}\label{app:classifier_rank_fig}
We present classifier-wise rank results in two complementary formats. Figure~\ref{fig:method_rank_heatmap} gives a compact visual overview (lower is better), while Table~\ref{tab:method_rank_per_classifier} reports the exact values for each method-classifier pair.

\begin{figure}[t]
\centering
\input{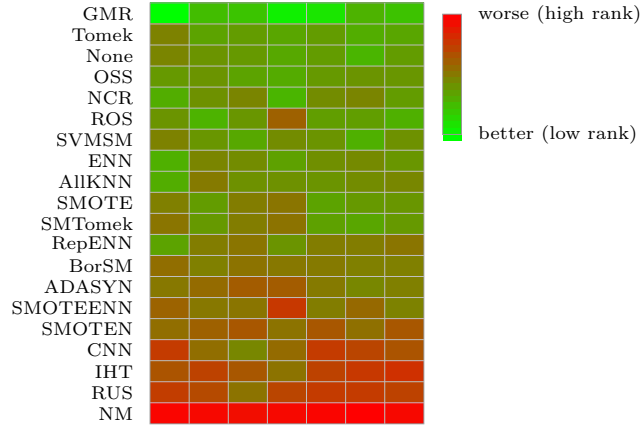}
\caption{Heatmap of average ranks across 7 classifiers (averaged over 27 datasets). Colors map low ranks (better) to green and high ranks (worse) to red.}
\label{fig:method_rank_heatmap}
\end{figure}

Table~\ref{tab:method_rank_per_classifier} lists the exact average ranks corresponding to the heatmap.
\begin{table}[t]
\centering
\caption{Average rank of each sampling method across 7 classifiers. Each cell shows the average rank over 27 datasets for a specific method-classifier combination. Lower ranks indicate better performance. The rightmost column shows the average rank across all classifiers. Bold values indicate the best (minimum rank) for each classifier.}
\label{tab:method_rank_per_classifier}
\scriptsize
\setlength{\tabcolsep}{1.8pt}
\renewcommand{\arraystretch}{0.88}
\resizebox{0.65\columnwidth}{!}{%
\begin{tabular}{@{}lrrrrrrrr@{}}%
\toprule
\textbf{Method} & \textbf{DT} & \textbf{RF} & \textbf{SVM} & \textbf{KNN} & \textbf{LR} & \textbf{XGB} & \textbf{GBM} & \textbf{Avg} \\
\midrule
GMR & \textbf{1.00} & \textbf{5.93} & \textbf{5.30} & \textbf{2.19} & \textbf{2.89} & 6.67 & \textbf{5.59} & \textbf{4.22} \\
Tomek & 10.15 & 7.74 & 8.33 & 7.52 & 8.30 & 6.96 & 7.56 & 8.08 \\
None & 10.00 & 8.81 & 8.44 & 7.44 & 8.26 & \textbf{6.44} & 8.33 & 8.25 \\
OSS & 8.56 & 8.81 & 7.70 & 7.11 & 8.48 & 8.81 & 8.70 & 8.31 \\
NCR & 6.96 & 9.07 & 9.96 & 6.52 & 9.44 & 9.93 & 8.22 & 8.59 \\
ROS & 8.89 & 6.59 & 8.74 & 12.63 & 8.19 & 8.19 & 6.89 & 8.59 \\
SVMSM & 10.19 & 8.70 & 7.33 & 9.44 & 8.85 & 6.89 & 9.11 & 8.64 \\
ENN & 6.74 & 9.93 & 9.37 & 7.96 & 9.33 & 9.70 & 8.74 & 8.82 \\
AllKNN & 6.96 & 10.67 & 9.11 & 9.04 & 8.81 & 9.52 & 9.81 & 9.13 \\
SMOTE & 10.41 & 8.37 & 10.52 & 11.07 & 7.78 & 8.41 & 8.59 & 9.31 \\
SMTomek & 11.04 & 8.44 & 10.52 & 11.22 & 8.00 & 7.52 & 8.44 & 9.31 \\
RepENN & 7.67 & 10.48 & 11.04 & 8.89 & 10.52 & 10.63 & 11.07 & 10.04 \\
BorSM & 11.67 & 10.44 & 11.37 & 10.93 & 10.70 & 10.30 & 10.44 & 10.84 \\
ADASYN & 10.85 & 11.81 & 13.00 & 12.96 & 10.67 & 9.74 & 10.41 & 11.35 \\
SMOTEENN & 12.41 & 11.00 & 11.15 & 15.59 & 10.63 & 12.11 & 10.26 & 11.88 \\
SMOTEN & 11.59 & 12.63 & 13.44 & 11.37 & 13.33 & 11.52 & 13.41 & 12.47 \\
CNN & 15.48 & 11.67 & 9.78 & 11.89 & 15.41 & 14.63 & 13.63 & 13.21 \\
IHT & 13.48 & 14.81 & 13.37 & 11.37 & 14.93 & 15.63 & 16.22 & 14.26 \\
RUS & 15.30 & 14.26 & 11.37 & 14.70 & 15.44 & 15.48 & 14.78 & 14.48 \\
NM & 19.33 & 19.26 & 18.81 & 19.26 & 19.37 & 19.74 & 19.19 & 19.28 \\
\bottomrule
\end{tabular}%
}
\begin{minipage}{\columnwidth}
\scriptsize
\textbf{Note:} Ranks are computed per classifier over 27 datasets. 
Bold values indicate best performance (minimum rank) for each column. 
Classifiers: DT=DecisionTree, RF=RandomForest, SVM, KNN, LR=LogisticRegression, XGB=XGBoost, GBM=GradientBoosting.\\
Method abbreviations are listed in Appendix~\ref{app:method_abbrev}.
\end{minipage}
\end{table}

\subsection{Complete Baseline Performance Comparison}\label{app:baseline_full}
This section provides the complete baseline performance comparison for all methods (including the None baseline which represents no sampling) across all datasets. Table~\ref{tab:baseline_performance_complete} presents the average AUPRC computed across seven classifiers (KNN, RF, LR, DT, SVM, GBM, XGBoost) for each method on each dataset. Methods are sorted by AUPRC within each dataset to facilitate comparison.

The \textit{None} method represents the performance of classifiers trained directly on the original imbalanced data without any sampling technique applied. This serves as the fundamental baseline for evaluating the effectiveness of various sampling methods.

\begin{longtable}{l l r}
\caption{Complete Baseline Performance: All Methods per Dataset. Average AUPRC across 7 classifiers (KNN, RF, LR, DT, SVM, GBM, XGBoost). Methods are sorted by AUPRC within each dataset. \textit{None} indicates no sampling (original imbalanced data).}\label{tab:baseline_performance_complete}\\
\toprule
\textbf{Dataset} & \textbf{Method} & \textbf{AUPRC} \\
\midrule
\endfirsthead
\multicolumn{3}{c}%
{\tablename\ \thetable\ -- \textit{Continued from previous page}} \\
\toprule
\textbf{Dataset} & \textbf{Method} & \textbf{AUPRC} \\
\midrule
\endhead
\midrule
\multicolumn{3}{r}{\textit{Continued on next page}} \\
\endfoot
\bottomrule
\endlastfoot
abalone & RepeatedEditedNearestNeighbours & 0.317395 \\
abalone & AllKNN & 0.308388 \\
abalone & EditedNearestNeighbours & 0.298878 \\
abalone & SMOTEENN & 0.290575 \\
abalone & NeighbourhoodCleaningRule & 0.289702 \\
abalone & OneSidedSelection & 0.276249 \\
abalone & InstanceHardnessThreshold & 0.276229 \\
abalone & TomekLinks & 0.275212 \\
abalone & RandomUnderSampler & 0.274275 \\
abalone & BorderlineSMOTE & 0.271323 \\
abalone & SVMSMOTE & 0.270217 \\
abalone & SMOTEN & 0.269107 \\
abalone & \textit{None} & 0.267599 \\
abalone & SMOTE & 0.265595 \\
abalone & SMOTETomek & 0.263242 \\
abalone & RandomOverSampler & 0.261236 \\
abalone & ADASYN & 0.257320 \\
abalone & CondensedNearestNeighbour & 0.252844 \\
abalone & NearMiss & 0.069666 \\
\addlinespace[0.5em]
abalone\_19 & RandomUnderSampler & 0.094961 \\
abalone\_19 & InstanceHardnessThreshold & 0.059524 \\
abalone\_19 & SMOTEENN & 0.045324 \\
abalone\_19 & SMOTE & 0.039893 \\
abalone\_19 & TomekLinks & 0.039015 \\
abalone\_19 & ADASYN & 0.038575 \\
abalone\_19 & RandomOverSampler & 0.034343 \\
abalone\_19 & SMOTETomek & 0.033677 \\
abalone\_19 & \textit{None} & 0.033603 \\
abalone\_19 & AllKNN & 0.032773 \\
abalone\_19 & NeighbourhoodCleaningRule & 0.032771 \\
abalone\_19 & RepeatedEditedNearestNeighbours & 0.032237 \\
abalone\_19 & EditedNearestNeighbours & 0.031769 \\
abalone\_19 & OneSidedSelection & 0.028238 \\
abalone\_19 & SVMSMOTE & 0.024391 \\
abalone\_19 & SMOTEN & 0.023041 \\
abalone\_19 & BorderlineSMOTE & 0.021823 \\
abalone\_19 & CondensedNearestNeighbour & 0.011083 \\
abalone\_19 & NearMiss & 0.009649 \\
\addlinespace[0.5em]
arrhythmia & ADASYN & 0.614053 \\
arrhythmia & \textit{None} & 0.611192 \\
arrhythmia & SMOTE & 0.609579 \\
arrhythmia & SMOTETomek & 0.609579 \\
arrhythmia & OneSidedSelection & 0.608412 \\
arrhythmia & RepeatedEditedNearestNeighbours & 0.599181 \\
arrhythmia & EditedNearestNeighbours & 0.595921 \\
arrhythmia & RandomOverSampler & 0.595600 \\
arrhythmia & TomekLinks & 0.595414 \\
arrhythmia & NeighbourhoodCleaningRule & 0.594395 \\
arrhythmia & SVMSMOTE & 0.592715 \\
arrhythmia & AllKNN & 0.591657 \\
arrhythmia & BorderlineSMOTE & 0.589470 \\
arrhythmia & SMOTEENN & 0.581444 \\
arrhythmia & SMOTEN & 0.551128 \\
arrhythmia & CondensedNearestNeighbour & 0.522678 \\
arrhythmia & RandomUnderSampler & 0.385158 \\
arrhythmia & InstanceHardnessThreshold & 0.351089 \\
arrhythmia & NearMiss & 0.326803 \\
\addlinespace[0.5em]
car\_eval\_34 & CondensedNearestNeighbour & 0.944469 \\
car\_eval\_34 & SMOTE & 0.943848 \\
car\_eval\_34 & SMOTETomek & 0.943848 \\
car\_eval\_34 & RandomOverSampler & 0.943310 \\
car\_eval\_34 & SMOTEN & 0.942174 \\
car\_eval\_34 & ADASYN & 0.940768 \\
car\_eval\_34 & SVMSMOTE & 0.938635 \\
car\_eval\_34 & NearMiss & 0.930633 \\
car\_eval\_34 & BorderlineSMOTE & 0.927220 \\
car\_eval\_34 & \textit{None} & 0.922286 \\
car\_eval\_34 & OneSidedSelection & 0.919299 \\
car\_eval\_34 & TomekLinks & 0.918873 \\
car\_eval\_34 & NeighbourhoodCleaningRule & 0.916354 \\
car\_eval\_34 & RepeatedEditedNearestNeighbours & 0.916000 \\
car\_eval\_34 & AllKNN & 0.914945 \\
car\_eval\_34 & EditedNearestNeighbours & 0.914775 \\
car\_eval\_34 & SMOTEENN & 0.911791 \\
car\_eval\_34 & InstanceHardnessThreshold & 0.908319 \\
car\_eval\_34 & RandomUnderSampler & 0.881307 \\
\addlinespace[0.5em]
car\_eval\_4 & CondensedNearestNeighbour & 0.946448 \\
car\_eval\_4 & RandomOverSampler & 0.945208 \\
car\_eval\_4 & SVMSMOTE & 0.940935 \\
car\_eval\_4 & ADASYN & 0.936142 \\
car\_eval\_4 & SMOTE & 0.929423 \\
car\_eval\_4 & SMOTETomek & 0.929423 \\
car\_eval\_4 & BorderlineSMOTE & 0.923705 \\
car\_eval\_4 & SMOTEN & 0.920780 \\
car\_eval\_4 & OneSidedSelection & 0.904499 \\
car\_eval\_4 & AllKNN & 0.891632 \\
car\_eval\_4 & TomekLinks & 0.891536 \\
car\_eval\_4 & \textit{None} & 0.891306 \\
car\_eval\_4 & NeighbourhoodCleaningRule & 0.885788 \\
car\_eval\_4 & EditedNearestNeighbours & 0.874307 \\
car\_eval\_4 & RepeatedEditedNearestNeighbours & 0.874307 \\
car\_eval\_4 & SMOTEENN & 0.873877 \\
car\_eval\_4 & InstanceHardnessThreshold & 0.871338 \\
car\_eval\_4 & RandomUnderSampler & 0.724921 \\
car\_eval\_4 & NearMiss & 0.398794 \\
\addlinespace[0.5em]
coil\_2000 & InstanceHardnessThreshold & 0.130768 \\
coil\_2000 & RandomUnderSampler & 0.127169 \\
coil\_2000 & NeighbourhoodCleaningRule & 0.124753 \\
coil\_2000 & RepeatedEditedNearestNeighbours & 0.124138 \\
coil\_2000 & EditedNearestNeighbours & 0.124020 \\
coil\_2000 & TomekLinks & 0.123720 \\
coil\_2000 & OneSidedSelection & 0.123490 \\
coil\_2000 & AllKNN & 0.123323 \\
coil\_2000 & \textit{None} & 0.122316 \\
coil\_2000 & CondensedNearestNeighbour & 0.122222 \\
coil\_2000 & SMOTEENN & 0.121458 \\
coil\_2000 & SVMSMOTE & 0.120430 \\
coil\_2000 & RandomOverSampler & 0.119438 \\
coil\_2000 & BorderlineSMOTE & 0.118154 \\
coil\_2000 & SMOTETomek & 0.117747 \\
coil\_2000 & ADASYN & 0.117205 \\
coil\_2000 & SMOTE & 0.116818 \\
coil\_2000 & SMOTEN & 0.102871 \\
coil\_2000 & NearMiss & 0.070939 \\
\addlinespace[0.5em]
ecoli & OneSidedSelection & 0.679323 \\
ecoli & \textit{None} & 0.670092 \\
ecoli & NeighbourhoodCleaningRule & 0.669095 \\
ecoli & TomekLinks & 0.666434 \\
ecoli & EditedNearestNeighbours & 0.660769 \\
ecoli & AllKNN & 0.655290 \\
ecoli & RandomUnderSampler & 0.647470 \\
ecoli & ADASYN & 0.637569 \\
ecoli & RandomOverSampler & 0.627547 \\
ecoli & SMOTE & 0.626608 \\
ecoli & SMOTETomek & 0.622792 \\
ecoli & SMOTEN & 0.620329 \\
ecoli & RepeatedEditedNearestNeighbours & 0.617065 \\
ecoli & SVMSMOTE & 0.615822 \\
ecoli & SMOTEENN & 0.615135 \\
ecoli & BorderlineSMOTE & 0.613431 \\
ecoli & CondensedNearestNeighbour & 0.611176 \\
ecoli & InstanceHardnessThreshold & 0.530402 \\
ecoli & NearMiss & 0.279424 \\
\addlinespace[0.5em]
isolet & TomekLinks & 0.856045 \\
isolet & OneSidedSelection & 0.855513 \\
isolet & \textit{None} & 0.854120 \\
isolet & NeighbourhoodCleaningRule & 0.850669 \\
isolet & RandomOverSampler & 0.846288 \\
isolet & EditedNearestNeighbours & 0.840449 \\
isolet & SVMSMOTE & 0.835170 \\
isolet & AllKNN & 0.834430 \\
isolet & SMOTE & 0.828889 \\
isolet & SMOTETomek & 0.828889 \\
isolet & SMOTEN & 0.824984 \\
isolet & RepeatedEditedNearestNeighbours & 0.822052 \\
isolet & CondensedNearestNeighbour & 0.820545 \\
isolet & ADASYN & 0.807127 \\
isolet & BorderlineSMOTE & 0.795150 \\
isolet & RandomUnderSampler & 0.776272 \\
isolet & SMOTEENN & 0.760503 \\
isolet & InstanceHardnessThreshold & 0.624215 \\
isolet & NearMiss & 0.387811 \\
\addlinespace[0.5em]
letter\_img & ADASYN & 0.934319 \\
letter\_img & NeighbourhoodCleaningRule & 0.934012 \\
letter\_img & RepeatedEditedNearestNeighbours & 0.933787 \\
letter\_img & EditedNearestNeighbours & 0.932822 \\
letter\_img & \textit{None} & 0.932303 \\
letter\_img & RandomOverSampler & 0.932142 \\
letter\_img & SVMSMOTE & 0.931972 \\
letter\_img & TomekLinks & 0.931817 \\
letter\_img & AllKNN & 0.931713 \\
letter\_img & OneSidedSelection & 0.931284 \\
letter\_img & SMOTE & 0.928943 \\
letter\_img & SMOTETomek & 0.928943 \\
letter\_img & BorderlineSMOTE & 0.924850 \\
letter\_img & SMOTEENN & 0.924516 \\
letter\_img & SMOTEN & 0.922953 \\
letter\_img & InstanceHardnessThreshold & 0.893851 \\
letter\_img & CondensedNearestNeighbour & 0.877469 \\
letter\_img & RandomUnderSampler & 0.805643 \\
letter\_img & NearMiss & 0.623239 \\
\addlinespace[0.5em]
libras\_move & BorderlineSMOTE & 0.896869 \\
libras\_move & SVMSMOTE & 0.881083 \\
libras\_move & RandomOverSampler & 0.880520 \\
libras\_move & ADASYN & 0.870014 \\
libras\_move & SMOTE & 0.850161 \\
libras\_move & SMOTETomek & 0.850161 \\
libras\_move & EditedNearestNeighbours & 0.849418 \\
libras\_move & NeighbourhoodCleaningRule & 0.848337 \\
libras\_move & OneSidedSelection & 0.845898 \\
libras\_move & \textit{None} & 0.834689 \\
libras\_move & TomekLinks & 0.834689 \\
libras\_move & AllKNN & 0.833356 \\
libras\_move & RepeatedEditedNearestNeighbours & 0.831513 \\
libras\_move & SMOTEENN & 0.824668 \\
libras\_move & SMOTEN & 0.797956 \\
libras\_move & InstanceHardnessThreshold & 0.719700 \\
libras\_move & CondensedNearestNeighbour & 0.666148 \\
libras\_move & RandomUnderSampler & 0.665363 \\
libras\_move & NearMiss & 0.517403 \\
\addlinespace[0.5em]
mammography & SVMSMOTE & 0.632852 \\
mammography & TomekLinks & 0.624105 \\
mammography & RandomOverSampler & 0.618372 \\
mammography & \textit{None} & 0.616847 \\
mammography & SMOTE & 0.606156 \\
mammography & SMOTETomek & 0.605682 \\
mammography & NeighbourhoodCleaningRule & 0.603693 \\
mammography & CondensedNearestNeighbour & 0.596846 \\
mammography & EditedNearestNeighbours & 0.593231 \\
mammography & BorderlineSMOTE & 0.592917 \\
mammography & AllKNN & 0.580091 \\
mammography & OneSidedSelection & 0.576425 \\
mammography & RepeatedEditedNearestNeighbours & 0.566129 \\
mammography & SMOTEENN & 0.563037 \\
mammography & ADASYN & 0.549686 \\
mammography & RandomUnderSampler & 0.510061 \\
mammography & SMOTEN & 0.503070 \\
mammography & InstanceHardnessThreshold & 0.305018 \\
mammography & NearMiss & 0.038770 \\
\addlinespace[0.5em]
oil & BorderlineSMOTE & 0.484308 \\
oil & SMOTE & 0.482234 \\
oil & ADASYN & 0.480444 \\
oil & SMOTETomek & 0.480316 \\
oil & SVMSMOTE & 0.477834 \\
oil & \textit{None} & 0.438855 \\
oil & NeighbourhoodCleaningRule & 0.437105 \\
oil & SMOTEN & 0.432911 \\
oil & TomekLinks & 0.430262 \\
oil & RandomOverSampler & 0.429126 \\
oil & OneSidedSelection & 0.428500 \\
oil & EditedNearestNeighbours & 0.427453 \\
oil & SMOTEENN & 0.421939 \\
oil & AllKNN & 0.396827 \\
oil & RepeatedEditedNearestNeighbours & 0.374801 \\
oil & InstanceHardnessThreshold & 0.318653 \\
oil & CondensedNearestNeighbour & 0.308418 \\
oil & RandomUnderSampler & 0.300908 \\
oil & NearMiss & 0.202147 \\
\addlinespace[0.5em]
optical\_digits & AllKNN & 0.946737 \\
optical\_digits & EditedNearestNeighbours & 0.946591 \\
optical\_digits & \textit{None} & 0.946543 \\
optical\_digits & NeighbourhoodCleaningRule & 0.945674 \\
optical\_digits & TomekLinks & 0.944721 \\
optical\_digits & RepeatedEditedNearestNeighbours & 0.944429 \\
optical\_digits & OneSidedSelection & 0.938495 \\
optical\_digits & SMOTE & 0.938098 \\
optical\_digits & SMOTETomek & 0.938098 \\
optical\_digits & RandomOverSampler & 0.937414 \\
optical\_digits & SMOTEENN & 0.935060 \\
optical\_digits & SVMSMOTE & 0.930706 \\
optical\_digits & SMOTEN & 0.929500 \\
optical\_digits & BorderlineSMOTE & 0.923988 \\
optical\_digits & ADASYN & 0.923797 \\
optical\_digits & RandomUnderSampler & 0.890297 \\
optical\_digits & CondensedNearestNeighbour & 0.849241 \\
optical\_digits & InstanceHardnessThreshold & 0.842486 \\
optical\_digits & NearMiss & 0.701474 \\
\addlinespace[0.5em]
ozone\_level & InstanceHardnessThreshold & 0.221960 \\
ozone\_level & RepeatedEditedNearestNeighbours & 0.198115 \\
ozone\_level & \textit{None} & 0.197729 \\
ozone\_level & EditedNearestNeighbours & 0.196054 \\
ozone\_level & OneSidedSelection & 0.195971 \\
ozone\_level & NeighbourhoodCleaningRule & 0.195696 \\
ozone\_level & AllKNN & 0.194429 \\
ozone\_level & RandomOverSampler & 0.194208 \\
ozone\_level & BorderlineSMOTE & 0.194146 \\
ozone\_level & SVMSMOTE & 0.194020 \\
ozone\_level & TomekLinks & 0.192708 \\
ozone\_level & SMOTEENN & 0.186460 \\
ozone\_level & SMOTETomek & 0.184983 \\
ozone\_level & CondensedNearestNeighbour & 0.184195 \\
ozone\_level & RandomUnderSampler & 0.182302 \\
ozone\_level & ADASYN & 0.177766 \\
ozone\_level & SMOTE & 0.176374 \\
ozone\_level & SMOTEN & 0.157684 \\
ozone\_level & NearMiss & 0.036128 \\
\addlinespace[0.5em]
pen\_digits & RandomOverSampler & 0.960806 \\
pen\_digits & TomekLinks & 0.960350 \\
pen\_digits & RepeatedEditedNearestNeighbours & 0.960302 \\
pen\_digits & \textit{None} & 0.960292 \\
pen\_digits & EditedNearestNeighbours & 0.959954 \\
pen\_digits & NeighbourhoodCleaningRule & 0.959947 \\
pen\_digits & OneSidedSelection & 0.959765 \\
pen\_digits & AllKNN & 0.959450 \\
pen\_digits & SMOTE & 0.957603 \\
pen\_digits & SMOTETomek & 0.957603 \\
pen\_digits & SMOTEENN & 0.956401 \\
pen\_digits & SMOTEN & 0.951931 \\
pen\_digits & InstanceHardnessThreshold & 0.923645 \\
pen\_digits & SVMSMOTE & 0.922007 \\
pen\_digits & RandomUnderSampler & 0.921658 \\
pen\_digits & ADASYN & 0.916603 \\
pen\_digits & BorderlineSMOTE & 0.908667 \\
pen\_digits & CondensedNearestNeighbour & 0.868471 \\
pen\_digits & NearMiss & 0.675025 \\
\addlinespace[0.5em]
protein\_homo & AllKNN & 0.839470 \\
protein\_homo & EditedNearestNeighbours & 0.838064 \\
protein\_homo & NeighbourhoodCleaningRule & 0.835545 \\
protein\_homo & OneSidedSelection & 0.834851 \\
protein\_homo & RepeatedEditedNearestNeighbours & 0.834632 \\
protein\_homo & TomekLinks & 0.832761 \\
protein\_homo & \textit{None} & 0.830796 \\
protein\_homo & RandomOverSampler & 0.818520 \\
protein\_homo & SMOTEN & 0.809633 \\
protein\_homo & InstanceHardnessThreshold & 0.788383 \\
protein\_homo & SVMSMOTE & 0.769511 \\
protein\_homo & CondensedNearestNeighbour & 0.763495 \\
protein\_homo & BorderlineSMOTE & 0.759034 \\
protein\_homo & SMOTETomek & 0.719719 \\
protein\_homo & SMOTE & 0.719372 \\
protein\_homo & ADASYN & 0.716366 \\
protein\_homo & SMOTEENN & 0.715738 \\
protein\_homo & RandomUnderSampler & 0.683711 \\
protein\_homo & NearMiss & 0.137922 \\
\addlinespace[0.5em]
satimage & \textit{None} & 0.581587 \\
satimage & OneSidedSelection & 0.580241 \\
satimage & TomekLinks & 0.578822 \\
satimage & NeighbourhoodCleaningRule & 0.575820 \\
satimage & RandomOverSampler & 0.566923 \\
satimage & SMOTE & 0.563886 \\
satimage & SMOTETomek & 0.563886 \\
satimage & EditedNearestNeighbours & 0.563377 \\
satimage & SVMSMOTE & 0.560030 \\
satimage & CondensedNearestNeighbour & 0.553217 \\
satimage & SMOTEN & 0.552696 \\
satimage & AllKNN & 0.552265 \\
satimage & ADASYN & 0.543886 \\
satimage & BorderlineSMOTE & 0.535283 \\
satimage & SMOTEENN & 0.525733 \\
satimage & RepeatedEditedNearestNeighbours & 0.525305 \\
satimage & RandomUnderSampler & 0.513260 \\
satimage & InstanceHardnessThreshold & 0.410938 \\
satimage & NearMiss & 0.121733 \\
\addlinespace[0.5em]
scene & AllKNN & 0.260258 \\
scene & EditedNearestNeighbours & 0.255101 \\
scene & SVMSMOTE & 0.254637 \\
scene & NeighbourhoodCleaningRule & 0.248582 \\
scene & OneSidedSelection & 0.248209 \\
scene & RepeatedEditedNearestNeighbours & 0.246716 \\
scene & TomekLinks & 0.241893 \\
scene & \textit{None} & 0.240611 \\
scene & BorderlineSMOTE & 0.238879 \\
scene & RandomOverSampler & 0.236289 \\
scene & ADASYN & 0.233983 \\
scene & SMOTE & 0.231977 \\
scene & SMOTETomek & 0.231817 \\
scene & RandomUnderSampler & 0.229056 \\
scene & SMOTEENN & 0.227890 \\
scene & SMOTEN & 0.224181 \\
scene & CondensedNearestNeighbour & 0.216484 \\
scene & InstanceHardnessThreshold & 0.200717 \\
scene & NearMiss & 0.102695 \\
\addlinespace[0.5em]
sick\_euthyroid & SVMSMOTE & 0.785684 \\
sick\_euthyroid & SMOTETomek & 0.784999 \\
sick\_euthyroid & RandomOverSampler & 0.784531 \\
sick\_euthyroid & SMOTE & 0.782663 \\
sick\_euthyroid & BorderlineSMOTE & 0.782143 \\
sick\_euthyroid & \textit{None} & 0.780421 \\
sick\_euthyroid & ADASYN & 0.779566 \\
sick\_euthyroid & TomekLinks & 0.778363 \\
sick\_euthyroid & NeighbourhoodCleaningRule & 0.777484 \\
sick\_euthyroid & OneSidedSelection & 0.775670 \\
sick\_euthyroid & EditedNearestNeighbours & 0.775141 \\
sick\_euthyroid & SMOTEN & 0.770853 \\
sick\_euthyroid & AllKNN & 0.765810 \\
sick\_euthyroid & RepeatedEditedNearestNeighbours & 0.763223 \\
sick\_euthyroid & SMOTEENN & 0.759671 \\
sick\_euthyroid & CondensedNearestNeighbour & 0.733857 \\
sick\_euthyroid & RandomUnderSampler & 0.696834 \\
sick\_euthyroid & InstanceHardnessThreshold & 0.676978 \\
sick\_euthyroid & NearMiss & 0.398005 \\
\addlinespace[0.5em]
solar\_flare\_m0 & CondensedNearestNeighbour & 0.160947 \\
solar\_flare\_m0 & NeighbourhoodCleaningRule & 0.156081 \\
solar\_flare\_m0 & AllKNN & 0.155723 \\
solar\_flare\_m0 & SMOTEENN & 0.154573 \\
solar\_flare\_m0 & RepeatedEditedNearestNeighbours & 0.151498 \\
solar\_flare\_m0 & EditedNearestNeighbours & 0.149907 \\
solar\_flare\_m0 & TomekLinks & 0.149051 \\
solar\_flare\_m0 & RandomUnderSampler & 0.147578 \\
solar\_flare\_m0 & InstanceHardnessThreshold & 0.146633 \\
solar\_flare\_m0 & OneSidedSelection & 0.144293 \\
solar\_flare\_m0 & SVMSMOTE & 0.142220 \\
solar\_flare\_m0 & RandomOverSampler & 0.141912 \\
solar\_flare\_m0 & \textit{None} & 0.141722 \\
solar\_flare\_m0 & BorderlineSMOTE & 0.139522 \\
solar\_flare\_m0 & SMOTEN & 0.138550 \\
solar\_flare\_m0 & SMOTE & 0.137186 \\
solar\_flare\_m0 & SMOTETomek & 0.137186 \\
solar\_flare\_m0 & ADASYN & 0.129831 \\
solar\_flare\_m0 & NearMiss & 0.100430 \\
\addlinespace[0.5em]
spectrometer & SMOTEN & 0.767142 \\
spectrometer & EditedNearestNeighbours & 0.766650 \\
spectrometer & RepeatedEditedNearestNeighbours & 0.766299 \\
spectrometer & \textit{None} & 0.765167 \\
spectrometer & AllKNN & 0.765159 \\
spectrometer & SVMSMOTE & 0.761152 \\
spectrometer & TomekLinks & 0.759670 \\
spectrometer & BorderlineSMOTE & 0.759096 \\
spectrometer & NeighbourhoodCleaningRule & 0.758629 \\
spectrometer & SMOTETomek & 0.750427 \\
spectrometer & SMOTE & 0.747447 \\
spectrometer & ADASYN & 0.740255 \\
spectrometer & RandomOverSampler & 0.739555 \\
spectrometer & SMOTEENN & 0.737696 \\
spectrometer & OneSidedSelection & 0.731471 \\
spectrometer & InstanceHardnessThreshold & 0.703745 \\
spectrometer & RandomUnderSampler & 0.698745 \\
spectrometer & CondensedNearestNeighbour & 0.647825 \\
spectrometer & NearMiss & 0.393815 \\
\addlinespace[0.5em]
thyroid\_sick & \textit{None} & 0.768265 \\
thyroid\_sick & TomekLinks & 0.767235 \\
thyroid\_sick & OneSidedSelection & 0.766953 \\
thyroid\_sick & RandomOverSampler & 0.766101 \\
thyroid\_sick & NeighbourhoodCleaningRule & 0.766076 \\
thyroid\_sick & ADASYN & 0.761758 \\
thyroid\_sick & SMOTETomek & 0.758488 \\
thyroid\_sick & SMOTE & 0.756963 \\
thyroid\_sick & AllKNN & 0.756619 \\
thyroid\_sick & SVMSMOTE & 0.755433 \\
thyroid\_sick & EditedNearestNeighbours & 0.753655 \\
thyroid\_sick & BorderlineSMOTE & 0.749557 \\
thyroid\_sick & RepeatedEditedNearestNeighbours & 0.745291 \\
thyroid\_sick & CondensedNearestNeighbour & 0.732923 \\
thyroid\_sick & SMOTEENN & 0.726312 \\
thyroid\_sick & SMOTEN & 0.721404 \\
thyroid\_sick & RandomUnderSampler & 0.658317 \\
thyroid\_sick & InstanceHardnessThreshold & 0.646062 \\
thyroid\_sick & NearMiss & 0.450881 \\
\addlinespace[0.5em]
us\_crime & SMOTEN & 0.460008 \\
us\_crime & NeighbourhoodCleaningRule & 0.457170 \\
us\_crime & TomekLinks & 0.452950 \\
us\_crime & OneSidedSelection & 0.451024 \\
us\_crime & RepeatedEditedNearestNeighbours & 0.450062 \\
us\_crime & SVMSMOTE & 0.448927 \\
us\_crime & \textit{None} & 0.448593 \\
us\_crime & AllKNN & 0.448385 \\
us\_crime & EditedNearestNeighbours & 0.445048 \\
us\_crime & CondensedNearestNeighbour & 0.431910 \\
us\_crime & RandomUnderSampler & 0.429545 \\
us\_crime & RandomOverSampler & 0.429413 \\
us\_crime & SMOTE & 0.420780 \\
us\_crime & SMOTETomek & 0.420780 \\
us\_crime & BorderlineSMOTE & 0.416748 \\
us\_crime & SMOTEENN & 0.409801 \\
us\_crime & ADASYN & 0.403203 \\
us\_crime & InstanceHardnessThreshold & 0.391133 \\
us\_crime & NearMiss & 0.310530 \\
\addlinespace[0.5em]
webpage & OneSidedSelection & 0.716597 \\
webpage & TomekLinks & 0.716552 \\
webpage & \textit{None} & 0.715568 \\
webpage & CondensedNearestNeighbour & 0.709512 \\
webpage & NeighbourhoodCleaningRule & 0.694978 \\
webpage & EditedNearestNeighbours & 0.685888 \\
webpage & AllKNN & 0.676386 \\
webpage & RandomOverSampler & 0.666281 \\
webpage & RepeatedEditedNearestNeighbours & 0.661258 \\
webpage & SVMSMOTE & 0.645779 \\
webpage & SMOTE & 0.628776 \\
webpage & SMOTETomek & 0.628776 \\
webpage & SMOTEN & 0.608370 \\
webpage & BorderlineSMOTE & 0.593482 \\
webpage & ADASYN & 0.579732 \\
webpage & SMOTEENN & 0.549320 \\
webpage & RandomUnderSampler & 0.511892 \\
webpage & InstanceHardnessThreshold & 0.508713 \\
webpage & NearMiss & 0.211698 \\
\addlinespace[0.5em]
wine\_quality & RandomOverSampler & 0.287078 \\
wine\_quality & \textit{None} & 0.283333 \\
wine\_quality & SVMSMOTE & 0.275223 \\
wine\_quality & TomekLinks & 0.274099 \\
wine\_quality & NeighbourhoodCleaningRule & 0.273313 \\
wine\_quality & OneSidedSelection & 0.271907 \\
wine\_quality & AllKNN & 0.270325 \\
wine\_quality & EditedNearestNeighbours & 0.268315 \\
wine\_quality & BorderlineSMOTE & 0.268253 \\
wine\_quality & RepeatedEditedNearestNeighbours & 0.263043 \\
wine\_quality & SMOTETomek & 0.260419 \\
wine\_quality & SMOTE & 0.258939 \\
wine\_quality & ADASYN & 0.252852 \\
wine\_quality & CondensedNearestNeighbour & 0.239557 \\
wine\_quality & SMOTEENN & 0.231595 \\
wine\_quality & SMOTEN & 0.230773 \\
wine\_quality & InstanceHardnessThreshold & 0.229103 \\
wine\_quality & RandomUnderSampler & 0.206871 \\
wine\_quality & NearMiss & 0.057979 \\
\addlinespace[0.5em]
yeast\_me2 & NeighbourhoodCleaningRule & 0.367299 \\
yeast\_me2 & EditedNearestNeighbours & 0.366031 \\
yeast\_me2 & SVMSMOTE & 0.364783 \\
yeast\_me2 & BorderlineSMOTE & 0.356850 \\
yeast\_me2 & RepeatedEditedNearestNeighbours & 0.356482 \\
yeast\_me2 & AllKNN & 0.356166 \\
yeast\_me2 & OneSidedSelection & 0.352041 \\
yeast\_me2 & TomekLinks & 0.350277 \\
yeast\_me2 & SMOTE & 0.347099 \\
yeast\_me2 & \textit{None} & 0.342411 \\
yeast\_me2 & SMOTETomek & 0.342317 \\
yeast\_me2 & ADASYN & 0.339819 \\
yeast\_me2 & SMOTEENN & 0.309819 \\
yeast\_me2 & RandomOverSampler & 0.309086 \\
yeast\_me2 & CondensedNearestNeighbour & 0.307005 \\
yeast\_me2 & InstanceHardnessThreshold & 0.289938 \\
yeast\_me2 & RandomUnderSampler & 0.265066 \\
yeast\_me2 & SMOTEN & 0.234466 \\
yeast\_me2 & NearMiss & 0.113107 \\
\addlinespace[0.5em]
yeast\_ml8 & InstanceHardnessThreshold & 0.120360 \\
yeast\_ml8 & RandomUnderSampler & 0.108407 \\
yeast\_ml8 & TomekLinks & 0.108024 \\
yeast\_ml8 & RepeatedEditedNearestNeighbours & 0.106284 \\
yeast\_ml8 & OneSidedSelection & 0.105521 \\
yeast\_ml8 & SMOTEENN & 0.104856 \\
yeast\_ml8 & SVMSMOTE & 0.104654 \\
yeast\_ml8 & BorderlineSMOTE & 0.104453 \\
yeast\_ml8 & AllKNN & 0.104293 \\
yeast\_ml8 & RandomOverSampler & 0.104158 \\
yeast\_ml8 & EditedNearestNeighbours & 0.103038 \\
yeast\_ml8 & \textit{None} & 0.101586 \\
yeast\_ml8 & SMOTEN & 0.100680 \\
yeast\_ml8 & NeighbourhoodCleaningRule & 0.099834 \\
yeast\_ml8 & ADASYN & 0.098387 \\
yeast\_ml8 & SMOTE & 0.097896 \\
yeast\_ml8 & SMOTETomek & 0.097896 \\
yeast\_ml8 & NearMiss & 0.093707 \\
yeast\_ml8 & CondensedNearestNeighbour & 0.088049 \\
\end{longtable}

\clearpage

\section{Dataset Coverage}\label{app:datasets}

To avoid ambiguity between canonical dataset families and benchmark-specific variants, we list the exact 27 dataset identifiers used in the performance tables.

\begin{table}[h]
\centering
\caption{Dataset identifiers used in the benchmark (27 total).}
\label{tab:dataset_ids}
\scriptsize
\begin{tabular}{l}
\toprule
\textbf{Dataset ID} \\
\midrule
\texttt{abalone}, \texttt{abalone\_19}, \texttt{arrhythmia}, \texttt{car\_eval\_34}, \texttt{car\_eval\_4}, \\
\texttt{coil\_2000}, \texttt{ecoli}, \texttt{isolet}, \texttt{letter\_img}, \texttt{libras\_move}, \\
\texttt{mammography}, \texttt{oil}, \texttt{optical\_digits}, \texttt{ozone\_level}, \texttt{pen\_digits}, \\
\texttt{protein\_homo}, \texttt{satimage}, \texttt{scene}, \texttt{sick\_euthyroid}, \texttt{solar\_flare\_m0}, \\
\texttt{spectrometer}, \texttt{thyroid\_sick}, \texttt{us\_crime}, \texttt{webpage}, \texttt{wine\_quality}, \\
\texttt{yeast\_me2}, \texttt{yeast\_ml8}. \\
\bottomrule
\end{tabular}
\end{table}

\end{document}